\documentclass[journal]{IEEEtran}

\usepackage{cite}
\usepackage{graphicx}
\usepackage{amsmath, amssymb,amsthm}
\usepackage[dvipsnames]{xcolor}
\usepackage{caption}
\usepackage{subcaption}
\usepackage{comment}
\usepackage{booktabs}
\usepackage[%
pagebackref=false,
bookmarks=true,bookmarksopen=true,
bookmarksnumbered=true,
breaklinks=true,
colorlinks=true,
anchorcolor=blue,citecolor=black,
urlcolor=blue,linkcolor=blue,filecolor=blue,
menucolor=blue,
]{hyperref}
\usepackage[normalem]{ulem}
\usepackage[ruled,vlined,linesnumbered]{algorithm2e}
\SetAlCapSkip{1em}

\SetCommentSty{mycommfont}

\newcommand{\R}[0]{\mathbf{R}}

\newtheorem{remark}{Remark}[section]
\newtheorem{theorem}{Theorem}[section]
\newtheorem{proposition}{Proposition}
\newtheorem{assumption}{Assumption}
\newtheorem{definition}[theorem]{Definition}
\newtheorem{problem}{Problem}[section]

\newcommand{\ppsout}{\bgroup\markoverwith{\textcolor{MidnightBlue}{\rule[0.5ex]{2pt}{0.4pt}}}\ULon}

\newcommand{\final}[1]{{\color{black} #1}}
\newcommand{\rag}[1]{{\color{magenta} \textrm{RKR: #1}}}

\DeclareMathOperator*{\minimize}{minimize\ }
\DeclareMathOperator*{\maximize}{maximize\ }

\newcommand{\etc}{\emph{etc.}\xspace}
\newcommand{\ie}{\emph{i.e.,}\xspace}
\newcommand{\eg}{\emph{e.g.,}\xspace}

\title{\LARGE \bf
Resilient Monitoring in Heterogeneous Multi-robot Systems \\ through Network Reconfiguration
}

\author{Ragesh K. Ramachandran$^1$, Pietro Pierpaoli$^2$, Magnus Egerstedt$^2$ and Gaurav S. Sukhatme$^1$ % <-this % stops a space
\thanks{$^1$Department of Computer Science, University of Southern California, Los Angeles, CA 90089, USA  (email: {\tt\small \{rageshku, gaurav\}@usc.edu}).}
\thanks{$^2$School of Electrical and Computer Engineering, Georgia Institute of Technology, Atlanta, GA, 30332 USA (email: {\tt\small \{pietro, magnus\}@gatech.edu}).} 
\thanks{This work was supported in part by the Army Research Laboratory as part of the Distributed and Collaborative Intelligent Systems and Technology (DCIST) Collaborative Research Alliance (CRA).}
}
\begin{document}
\bstctlcite{IEEEexample:BSTcontrol}

\markboth{T-RO SPECIAL ISSUE ON RESILIENCE IN NETWORKED ROBOTIC SYSTEMS}{Ramachandran \MakeLowercase{\textit{et al.}}}

\maketitle
\begin{abstract}
We propose a framework for resilience
%optimal task performance 
in a networked heterogeneous multi-robot team subject to resource failures. Each robot in the team is equipped with resources that it shares with its neighbors, \final{which are identified based on the team's communication graph}.  Additionally, each robot in the team executes a task, whose performance depends on the resources to which it has access. When a resource on a particular robot becomes unavailable (\eg a camera ceases to function), the team optimally 
reconfigures its communication network so that the robots affected by the failure can continue their tasks. 
We focus on 
%For concreteness, we consider 
a monitoring task, where robots individually estimate the state of an exogenous process. We encode the end-to-end effect of a robot's resource loss on the monitoring performance of the team by defining a new stronger notion of observability -- \textit{one-hop observability}. By abstracting the impact that {low-level} individual resources have on the task performance through the notion of one-hop observability, our framework leads to the principled reconfiguration of information flow in the team to effectively replace the lost resource on one robot with information from another, as long as certain conditions are met. 
%the new resources distribution 
%guarantees optimal performance 
%{allows robots to complete their tasks} (e.g. failed lidar replaced with {a sensor combination consisting of a camera and a ultrasound sensor}). 
Network reconfiguration is converted to the problem of selecting edges to be modified in the system's communication graph after a resource failure has occurred. A controller based on finite-time convergence control barrier functions drives each robot to a spatial location that enable the communication links of the modified graph. We validate the effectiveness of our framework by deploying it on a team of differential-drive robots estimating the position of a group of quadrotors.
\end{abstract}

\begin{IEEEkeywords}
Multi-Robot Systems
\end{IEEEkeywords}

%\input{sections/introduction.tex}

% the main tex file
\section{Introduction}
\label{sec: intro}

\IEEEPARstart{A}{dvancements} in embedded systems resulting in cheap and reliable computing devices, efficient actuators and sensors, and commodity networking have paved the way toward the development of multi-robot teams conceived for varied tasks such as reconnaissance \cite{Quann17}, environment monitoring \cite{Wang19}, and target tracking \cite{Zhou19}. 

The problem of designing a multi-robot team to achieve a given task is ill-posed. Even with some constraints in place (\eg number of robots, capabilities available on each robot, cost, \etc) many degrees of freedom remain open to the designer. A choice that simplifies the design space considerably is to compose a team of homogeneous robots. Alternatively, a heterogeneous multi-robot team composed of {\it specialized} robots each with a different set of sensing, actuation, communication and computational capabilities, holds the promise of reducing unnecessary redundancies (and therefore, potentially cost). For example, if a team of robots is designed for a construction task, equipping every robot in the team with all the capabilities required for the task is impossible or prohibitively expensive. \final{As noted in ~\cite{stone2000multiagent}, certain tasks can be accomplished by a heterogeneous multi-robot team where their expensive homogeneous counterparts fail.}
\final{At the same time, there are numerous challenges associated with the design of heterogeneous multi-robot systems  } ~\cite{Harvey16,parker2016multiple,Twu14,doi:10.1177/0278364906065378}.
%\gav{We need to enumerate the top 3-4 challenges here, and give an example of a paper for each}
These include but are not restricted to: distributing capabilities among the team members  ~\cite{parker2003effect}, efficiently decomposing the overall task into subtasks and assigning them to the robots in the team according to their capabilities \cite{Prorok16}, developing languages and associated symbols for lucid communication among robots with distinct capabilities \cite{coradeschi2013short}, designing strategies to mitigate the effect of robot resource failures towards team task performance  \cite{ramachandran2019resilience}.

 %While designing a multi-robot team to perform a particular task, the designer needs to arrive at a decision regarding the various capabilities that has to be attributed to the robots in the team.
 %{A straight forward choice would be to compose team with robots with the same capabilities, which results in a ``homogeneous'' multi-robot team. }
%{Alternatively, one can design and build a ``heterogeneous'' multi-robot team where robots have varied sensing, actuation, communication and computational capabilities to perform the same task.} 

%\rag{Hence, for the construction task it is natural to employ a team robot with different capabilities.} \ppnote{[Maybe this sentence can be removed.]}%\ppnote{[for example? stress also on the fact that building robots that can do everything is anyway impossible, or extremely expansive.]} 
%Moreover, designing heterogeneous robot team may offer economic benefits, since rather than building duplicate copies of expensive proficient robots it can be cheaper to distribute varying capabilities among the inexpensive robots in the team. 

We focus on one question arising in heterogeneous team design  -- {\em how can robots in a heterogeneous team work together by sharing resources, in order to be resilient as a team to failures of individual components?} For example, when a particular sensor on a robot fails (\eg a high resolution camera malfunctions on a robot that flies at high altitude), it may be able to rely on measurements made by a teammate nearby (\eg perhaps one that has a lower resolution camera but flies at a lower altitude).  In the heterogeneous setting, when a fault occurs, appropriate techniques must be in place to reconfigure the underlying interaction network between the robots or redistribute roles so that each robot has access to the resources it needs within the set of robots it can communicate with. Importantly, it is the objective of the sub-task assigned to the robot experiencing the failure that dictates the resource the failed robot needs to access from its neighbors. We note that in a homogeneous setting, any teammate will have the same sensing capabilities as the stricken robot. While this does not make the problem of resource sharing trivial (some robots will have the right vantage point so their data can be meaningfully used by the stricken robot, while others will not) it does simplify matters significantly. 

We formalize the intuitive notion of using a heterogeneous set of robots to construct a resilient team that can reconfigure its communication network to continue performing its overall task. \final{In our formulation, resilience is the ability of robots to perform tasks even after experiencing some failures. This notion is different from robustness which is the ability of systems to reject external disturbances.}
%However, designing heterogeneous systems come with its own challenges such as  control design, task allocation, quantifying heterogeneity, to name a few. Consequently, these factors have instigated  
%{while the modified communication topology is maintained close to the original communication topology} 
%{while minimally modifying the pre-failure spatial configuration of the robots}.  
We consider a specific multi-robot task -- collectively monitoring an exogenous process of interest. For example, the team may be tasked with monitoring the motion of targets of interest within a geospatial fence. The exogenous process is the ensemble of target positions each evolving as a function of time. Each robot in the team is equipped with sensors and {the exogenous process is assumed to be collectively observable by the} team. For example, one robot on the team may have a high resolution camera, inertial sensing and a significant amount of local processing allowing it to jointly estimate its pose and the poses of any targets it detects within the prescribed fence with high accuracy. Another robot may have a low resolution camera, inertial sensing and minimal local processing allowing it to jointly estimate its poses and the poses of targets within a smaller area and with lower accuracy. %{is assumed to be collectively observable allowing them to collectively estimate the state of the exogenous process}. %Note that, since the robots are entrusted with a monitoring task, we use the terms ``resource'' and ``sensor'' interchangeably in this manuscript. 
{%In order to expedite the estimation process
We think of the problem in the following terms. Each robot in the team is entrusted with the subtask of estimating the state of the process of interest using resources available within its immediate communication neighborhood (\eg a robot can share its sensor data within its communication neighborhood). Our approach is to enforce a strong observability condition on each robot which we call \textit{one-hop observability} ({OHO}).} We say that a robot is one-hop observable if the subsystem consisting of the robot and its {immediate} neighbors is observable. \final{A multi-robot system is referred to as \emph{one-hop observable} if each robot is the system is one-hop observable. This notion of observability is stronger than \emph{collective observability}. In other words, a one-hop observable multi-robot system is also collectively observable, but the converse may not necessarily hold. } In order to guarantee post-failure OHO, we propose an algorithm built on a \textit{mixed integer semi-definite program} (MISDP) which optimizes a team task performance metric based on OHO to generate the new communication graph for team in response to a failure. 
Specifically, our MISDP minimizes the average over the trace of the inverse of the robots’ one-hop observability Gramians.

 %Once the new graph is computed, robots execute a controller based on finite-time convergence control barrier function which drives them to a configuration that allows the desired inter-agent communication to occur. \ppnote{[This last part is repeating the info in the previous paragraph. Maybe can be combined.]}

%\rag{In this paper, we attempt to formalize the notion of resilience by exploiting the heterogeneity of the team with regard to the task executed by the robots in the team. We achieve this objective using a two-step strategy.  In response to a resource failure, our method first updates the communication network the robots use to share resources, such that the team is guaranteed to achieve the optimal task performance achievable with the available resources.}

The overarching idea presented here is to formalize the notion of resilience by exploiting the heterogeneity of the team with regard to the monitoring task executed by the robots in the team. The notion of one-hop observability provides a means to quantify the impact of each sensor in the multi-robot team towards the monitoring task performance. The one-hop observability notion allows the stricken robot to rely on a different sensor on a neighboring robot (or a combination of different sensors) to compensate for the performance loss caused due to the sensor failure. 

The methodology described here is a two-step strategy.  In response to a resource failure, our method first updates the communication network the robots use to share resources, such that the team is guaranteed to attain monitoring task performance commensurate with the available resources. Following this, the robots execute a controller based on finite-time convergence control barrier functions \cite{pierpaoli2019sequential} designed to drive them to a spatial configuration that results in the desired new communication links being formed.

Although in this article (an earlier version appears in~\cite{ramachandran2019resilience}) we restrict our attention to the monitoring task, in the future, we plan to extend the ideas discussed here to other multi-robot tasks where the task performance metric can be expressed as a function of the multi-robot team's resource distribution. This article makes the following contributions: 
\begin{itemize}
    %\item We developed a framework to formalize the idea that, in the event of a robot resource failure, the heterogeneity in the multi-robot team can be exploited to drive them to a new configuration such that the desired task performance subject to the available resources can be achieved.
    \item We formally pose the problem of resilience in a heterogeneous robot team engaged in monitoring an exogenous process through inter-robot sensor sharing. 
    \item We introduce a new strong notion of observability - \textit{one-hop observability} - and exploit a metric based on this notion to drive network reconfiguration in response to sensor failure.
    \item We propose an algorithm to generate and assemble a communication graph that maximizes the post-failure performance of each robot, by guaranteeing one-hop observability. 
    \item We validate the effectiveness of our proposed framework through multi-robot experiments.
\end{itemize}
\final{In prior work \cite{ramachandran2019resilience}, we proposed a method to maintain high resource availability in a networked heterogeneous multi-robot  system subject to resource failures. In contrast,  the work presented here focuses on a strategy which enables each robot in a team to monitor an exogenous process  through inter-robot sensor sharing. In \cite{ramachandran2019resilience}, we used a probabilistic optimization technique to generate the spatial coordinates for the robots, whereas in this work the robots execute a feedback controller based on finite time convergence control barrier functions to drive them to a spatial configuration that relies the new communication graph in space. Additionally, we relax the constraint that robots need access to a pre-specified list of resource. Finally, we note that the robots in this work are performing a concrete task (monitoring), whereas in our prior work the task was purely notional.}
%This paper is organized as follows: \autoref{sec: related wrk} surveys the literature relevant to the paper. The notations and mathematical concepts required to understand the paper can be found in \autoref{sec: notations}. We formally describe the problem discussed in this article in \autoref{sec: problem}. \autoref{sec: opt reconfig} and \autoref{sec: config assemble} delineates our solution to the problems introduced in \autoref{sec: problem}. \autoref{sec: application} demonstrates the application of our approach to a convoy protection scenario. Finally, we conclude the paper with our concluding remarks in \autoref{sec: conclude}.

\section{Related work}
\label{sec: related wrk}
% Role-specialization in ecology

% Heterogeneity as Diversity
Inspired by the notions of biodiversity~\cite{walker1992biodiversity} and role-specialization~\cite{beshers2001models} typical of biological ecosystems, attempts have been made to formally quantify heterogeneity in robotic systems. For instance, \cite{balch2000hierarchic,Twu14,abbas2014characterizing} explore metrics to measure heterogeneity through the notion of entropy. While, these works consider heterogeneity as a measure of {\it diversity}, our focus is on the impact this diversity has on the {\it performance} of the team in executing a certain task, particularly how to leverage it for resilience. 

% Heterogeneous Multi-Robot Task Assignment
The advantage of heterogeneity in multi-robot systems is the possibility of being able to guarantee a desired level of performance by distributing tasks to role-specialized agents, thus avoiding unnecessary overly skilled robots. The sharing of resources between robots makes it possible to create sub-teams of interacting individuals, whose collective capabilities could not be achieved by a single robot~\cite{Parker06,Vig06}. However, the assignment problems (\eg what and with whom a resource must be shared) are combinatorial and computationally challenging. Additionally, the application of standard optimization-based techniques is challenged by the lack of computationally amenable representations for the mapping between robots' resources and the reward attainable from their assignment to a task. In order to alleviate these challenges, it is possible to consider macroscopic models of heterogeneity that describe resources across a team as continuous distributions~\cite{ pinciroli2016buzz,prorok2016fast}. For example, in~\cite{prorok2016adaptive} heterogeneous robots are assigned to tasks by designing a stochastic transition process over the different tasks such that the sum of resources over each task meets a certain demand. A similar formulation for the resource-to-task assignment is considered in \cite{ravichandar2020strata} where different robots are selected from species that accommodate stochastic variations across the resources of their individuals. Alternatively, it is possible to considered a representation of the heterogeneity in a team of robots by introducing {\it role-specializations} that describe the degree of suitability of a robots in executing a set of tasks~\cite{okamoto2008impact,li2003diversity}. For example, in~\cite{notomista2019optimal} a set-valued optimization framework is designed in order to assign heterogeneous robots to tasks, while taking into account their physical constraints and specializations. In contrast with these previous works, here we take a step towards the definition of end-to-end assignment problems (which include resource reconfiguration as special case), designed to match low-level robots' resources (in particular, sensors) with the requirements of the tasks. \final{Note that in this article, we do not  explicitly consider the effect of utility, quality or cost \cite{Gerkey2004,Korsah2013}  of the robots' resources on the overall team task. An investigation along these lines is potential future work.}

% Resilience 
\final{In addition, the ability to share resources can be leveraged in multi-robot systems to improve their resilience, as the loss of a resource in one robot can be compensated by the availability of the same resource in the team.}
%\sout{An additional advantage of heterogeneous multi-robot systems is the ability to share resources as a means of increasing their resilience.} 
The problem of resilience in multi-robot systems has been studied primarily in the context of constructing communication networks that preclude malicious robots in the network from exerting an adverse influence on the network ~\cite{zhang2012robustness, Zhang2015}. The work in ~\cite{guerrero2017formations}, focuses on building resilient robot networks based on a notion of resilience called \textit{r-robustness}~\cite{LeBlanc2013}. \final{On the contrary, our work focuses on countering the impact of the robot resource failures on multi-robot team performance.} In the work~\cite{luo2019minimum} a strategy to guarantee connectivity robustness while taking into account the robots' objective is proposed. The line of work presented in \cite{schlotfeldt2018resilient, Zhou19, tzoumas2018resilient} considers task performance optimization by exploiting the submodularity property of the task performance metric to arrive at suboptimal solutions for incorporating resilience in the multi-robot team performing the task. An alternate perspective was considered in~\cite{Wehbe18,Wehbe19}, where the authors formulate  combinatorial optimization problems to maximize the probability of security of a multi-robot system. The work defines the security of a multi-robot system using the control-theoretic concept of left invertiblility and uses a binary decision diagram to compute the probability of security of a multi-robot system. \final{In contrast to these works, we propose a different combinatorial optimization (MISDP) based scheme that reconfigures the communication network of a heterogeneous  multi-robot system in the event of a robot resource failure (e.g. sensor or actuator), enabling the multi-robot system to perform an entrusted task with its available resources.}

% Graph reconfiguration in robot teams
Finally, the idea of controlling robots' motion so that certain properties of the underlying communication graph are respected has been studied in different contexts. For example, connectivity-preserving controllers have been designed using edge weight functions~\cite{ji2007distributed}, estimates of graph spectral properties~\cite{sabattini2013distributed}, passivity-based methods ~\cite{igarashi2009passivity}, and control barrier functions~\cite{wang2016multi}. However, depending on the task that the robots have to complete, stronger requirements might be necessary. For instance, motivated by the rich literature on coordinated multi-robot behaviors~\cite{cortes2017coordinated}, the work in~\cite{li2018formally} proposed a formally correct strategy based on control barrier functions to achieve a specific interaction structure required by the particular behavior. The problem of synthesizing multi-robot interaction networks designed to respect a desired distribution of resources is discussed in~\cite{abbas2014characterizing}. In~\cite{ramachandran2019resilience} the idea was extended to the design of a network reconfiguration strategy so to achieve resilience to faults. Here, we exploit the heterogeneity in a multi-robot team with the objective of restoring team performance after occurrence of faults, while optimizing the performance of a monitoring task. This paper fits the line of work on characterization of post-failure capabilities in a robot team. %\rag{ \sout{, which is an interesting, yet, poorly understood concept}}. 

\section{Notation and preliminaries}
\label{sec: notations}

We use \final{small or capital letters (with or without subscripts)} to represent scalars. Bold small letters and bold capital letters to denote vectors and matrices, respectively. Calligraphic symbols are used to represent sets. For any positive integer $z \in \mathbb{Z}^+$, $[z]$ denotes the set $\{1,2, \cdots, z\}$. $\|\cdot\|$ denotes the standard Euclidean 2-norm and the induced 2-norm for vectors and matrices respectively. $\|\mathbf{M}\|_F$ denotes the Frobenius norm \final{\cite[chapter 5]{Horn:1985:MA:5509}} of a matrix $\mathbf{M} \in \R^{m_1 \times m_2}$, $\|\mathbf{M}\|_F \triangleq \sqrt{trace(\mathbf{M}^T\mathbf{M})}$. We use $\mathbf{1}^{m_1}$ and $\mathbf{1}^{m_1 \times m_2}$ to represent a vector and matrix of ones of appropriate dimension, respectively. Similarly, $\mathbf{0}^{m_1}$ and $\mathbf{0}^{m_1 \times m_2}$ denote a vector and matrix of zeros respectively. $|\cdot|$ yields the number of elements in a set and length of a vector when applied on a set and a vector, respectively. For any vector $\mathbf{t} \in \R^{m_1}$, $Diag(\mathbf{t})$ denotes a matrix with the elements of $\mathbf{t}$ along its diagonal. Also, $diag(\mathbf{M})$ outputs a vector which contains the diagonal entries of matrix $\mathbf{M}$ as its elements. $[\mathbf{M}]_{i,j}$ denotes the $i,j$ entry of $\mathbf{M}$. Furthermore, $\mathbf{M}_1 \oplus \mathbf{M}_2$ gives the block diagonal matrix containing the matrices. $\mathcal{S}^m_+$ denotes the space of $m \times m$ symmetric positive semi-definite matrices. Alternatively, $\mathbf{M} \succeq 0$ also indicate the $\mathbf{M}$ is positive semi-definite. Note that we use $\mathbf{I}_n$ to represent the $n\times n$ identity matrix.
A weighted undirected graph $\mathcal{G}$ is defined by the triplet $(\mathcal{V},\mathcal{E} \subseteq \mathcal{V} \times \mathcal{V},\mathbf{A} \in \{0,1\}^{|\mathcal{V}| \times |\mathcal{V}|})$,  where $\mathbf{A}$  is the unweighted adjacency matrix \final{\cite[chapter 2]{Mesbahi_Magnus}} of the graph. Also, $\overline{\mathcal{E}} = (\mathcal{V} \times \mathcal{V}) \setminus \mathcal{E}$ denotes the edge complement of $\mathcal{G}$. The \textit{graph Laplacian matrix} \final{\cite[chapter 2]{Mesbahi_Magnus}} $\mathbf{L}$ of $\mathcal{E}$ can be computed as
\begin{align}
	\label{eqn:laplacian defin}
	\mathbf{L} = Diag(\mathbf{A}\cdot \mathbf{1}^{|\mathcal{V}|}) - \mathbf{A}.
\end{align}
	We summarize some properties of a graph Laplacian matrix~\cite{GodsilRoyle2001} used in this article:
\begin{align}
	\label{eqn:lap property 1}
	\mathbf{L} \cdot \mathbf{1}^{|\mathcal{V}|} &= \mathbf{0}^{|\mathcal{V}|} \\
	\label{eqn:lap property 2}
	Trace(\mathbf{L}) &= \sum_{1\leq i < j \leq |\mathcal{V}|} [\mathbf{A}]_{i,j}.
\end{align}
Note that the above relationships are true for the weighted adjacency matrix also. If $\lambda_c(\mathcal{G})$ denotes the algebraic connectivity of the $\mathcal{G}$ {(\eg the second smallest eigenvalues of $L$)}, then $\mathcal{G}$ is connected if and only if
\begin{align}
    \label{eqn:property cnnct}
   \lambda_c(\mathcal{G}) > 0.
\end{align}
{Finally, given two graphs $\mathcal{G}_1(\mathcal{V},\mathcal{E}_1)$ and $\mathcal{G}_2(\mathcal{V},\mathcal{E}_2)$ defined on the same set of vertices, we say that $\mathcal{G}_1$ spans $\mathcal{G}_2$, or $\mathcal{G}_1 \subseteq \mathcal{G}_2$ in short, if $\mathcal{E}_1 \subseteq \mathcal{E}_2$.} \final{The Laplacian associated with every connected graph has only one zero eigenvalue with  $\mathbf{1}$ as its corresponding eigenvector. In addition, it can be shown that an undirected graph is connected if and only if $\mathbf{L} + \frac{1}{n} \mathbf{1} \mathbf{1}^T$ is a positive-definite matrix~\cite[Proposition 1]{sundin2017connectedness}.} %For ease of reference, the significant symbols used in this paper and their descriptions are listed in \autoref{tab:notation}.

\section{Problem formulation}
\label{sec: problem}

Consider a team of $N$ robots in $d$-dimensional space, whose spatial configuration is denoted with ${\bf x}(t) \in \mathbb{R}^{dN}$, where ${\bf  x}_i \in \mathbb{R}^d$ is the position of robot $i$. \final{ Given a configuration ${\bf x}(t)$, the interaction structure between robots is described by a $\Delta$-disk graph $\mathcal{G}({\bf x}) = \mathcal{G}(\mathcal{V},\mathcal{E}({\bf x}))$
\begin{equation} \label{eq:proximitygraph}
    \mathcal{E}({\bf x}) = \{ (i,j) \in \mathcal{V} \times \mathcal{V} \, | \, \| {\bf x}_i(t) - {\bf x}_j(t) \| \leq \Delta \in \mathbf{R}^+ \}
\end{equation}
where $\mathcal{V}$ is the node set representing the robots and $\mathcal{E}({\bf x})$ is the edge set induced by the robots' configuration ${\bf x}(t)$. The $\Delta$-disk graph is also referred as $\Delta$-\textit{interaction graph} or simply \textit{interaction graph} for brevity. The \textit{communication graph} of the robot team is a subgraph of its interaction graph. } If $\mathcal{N}(i)$ denotes the set of neighbors of robot $i$ according to the graph $\mathcal{G}({
\bf x})$, then we define $\mathcal{\Bar{N}}(i) \triangleq \mathcal{N}(i) \cup i$ and refer to it as the closed neighbors of robot $i$. Similarly, we define closed adjacency matrix $\mathbf{\Bar{A}} \triangleq \mathbf{{A}} + \mathbf{I}$, where $\mathbf{{A}}$ is the unweighted adjacency matrix associated with $\mathcal{G}$. We let the motion of robot $i$ be described by the control-affine dynamics 
\begin{equation} \label{eq:affinecontrol}
\dot{\mathbf{x}}_i = \mathbf{f}_i(\mathbf{x}_i) + \mathbf{g}_i(\mathbf{x}_i)\,\mathbf{u}_i, 
\end{equation}
where $\mathbf{f}_i$ and $\mathbf{g}_i$ are Lipschitz continuous vector fields that describe the autonomous and input-dependent dynamics respectively and $\mathbf{u}_i \in \mathcal{U}_i$ is the input to robot $i$. 
%\todo{General representation of the task the robots are attempting to solve} 

We assume that there are $r$ different resources in the multi-robot team. The set $[r]$ contains labels of the resources available within the heterogeneous team. As in~\cite{ramachandran2019resilience,abbas2016deploying},  we define a binary matrix $\{0,1\}^{n \times r}$,  which we refer to as the \textit{resource matrix} and denote by $\boldsymbol{\Gamma}$. The entries of $\boldsymbol{\Gamma}$ are assigned as follows, 
	
\begin{align}
	\label{eqn:resource matrix}
	[\boldsymbol{\Gamma}]_{i,j} =  
	\begin{cases}
	1 & ~ \text{robot } i \text{ has resource}~j  \in [r] \\
	0 & \text{otherwise.}
	\end{cases}
\end{align}

Note that, the tuple $(\mathcal{G}, \mathbf{x}, \boldsymbol{\Gamma})$ uniquely determines both  spatial and network representation of the heterogeneous multi-robot system.  Hence, we term the tuple $(\mathcal{G}, \mathbf{x}, \boldsymbol{\Gamma})$ as a \textit{configuration} of the heterogeneous multi-robot team and denote it by $\mathcal{C}$. We refer to a resource matrix $\boldsymbol{\Gamma}$ as a \textit{feasible resource matrix}, if {there exists a spatial configuration such that} the multi-robot team is capable of performing the required task using the corresponding resource distribution. The resource matrix is termed \emph{infeasible} if it is not feasible.  {Moreover, $\mathcal{\Bar{C}}$ denotes the space of all configurations which contain a feasible resource matrix. Let each robot in the multi-robot team be assigned a subtask and the subtask performance of robot $i$ be quantified using the metric $J_i(\mathcal{C},T)$. %$J_i(\mathcal{C},[T^l_i, T^h_i])$, where $[T^l_i, T^h_i],\, T^l_i < T^h_i$ is the time interval for which robot $i$ is entitled to perform the subtask.
Consequently, we define %$J(J_1(\mathcal{C},[T^l_1, T^h_1]), J_2(\mathcal{C},[T^l_2, T^h_2]), \dots, J_N(\mathcal{C},[T^l_N, T^h_N]), T)$ or $J(J_1, J_2, \dots, J_N,\mathcal{C}, T)$ 
$J(J_1(\mathcal{C},T), J_2(\mathcal{C},T), \dots, J_N(\mathcal{C},T), T)$ or $J(J_1, J_2, \dots, J_N,\mathcal{C}, T)$
as the task performance metric associated with multi-robot team task, where $T$ %$T \geq \max \{T^h_1-T^l_1, T^h_2-T^l_2, \dots, T^h_N-T^h_N\}$ 
is the time required by multi-robot team to execute the task. We show in \autoref{sec: opt reconfig} that, the monitoring task considered in the paper can be {represented by} a performance metric having the structure mentioned above.} %\ppnote{[Do we need a time interval associated with the task? It seems an unnecessary complication, do you have some extension in mind?]}

We model a random \textit{resource failure} by a map that consumes a resource matrix (excluding the $n \times r$ matrix of zeros) and generates a resource matrix by setting a random nonzero entry of the consumed resource matrix to zero. A \textit{tolerable resource failure} maps a feasible resource matrix to another feasible resource matrix, while a \textit{catastrophic resource failure} maps a feasible resource matrix to an infeasible resource matrix. If the multi-robot team is capable of performing its task in the face of a tolerable resource failure without changing its configuration then the failure is said to be a \textit{benign resource failure}. We assume that the robots in the team use techniques existing in literature on multi-robot fault detection for  resource failure identification~\cite{pierpaoli2018fault,Ghasemi2017}.

Consider a sequence of countably infinite resource failures $\mathcal{F} = [\xi_1, \xi_2, \cdots \xi_{\infty}]$ acting on a feasible resource matrix $\boldsymbol{\Gamma}$ sequentially. Now, let $\mathcal{F}_{n_f}=[\xi_1, \xi_2, \cdots, \xi_k, \cdots, \xi_{n_f}]$ be the first $n_f \in \mathbb{Z}^+$ tolerable resource failures in $\mathcal{F}$, such that $\xi_{n_f + 1} \in \mathcal{F}$ is a catastrophic failure. The $k^{th}$ tolerable resource failure in $\mathcal{F}_{n_f}$ is denoted as $\xi_k$. We denote $\boldsymbol{\Gamma}_r[k]$ as the resultant resource matrix after first $k$ tolerable resource failures in $\mathcal{F}_{n_f}$ acted on $\boldsymbol{\Gamma}$. Mathematically, $\boldsymbol{\Gamma}[k] = \xi_k(\xi_{k-1}(\xi_{k-2}(\cdots \xi_1(\boldsymbol{\Gamma}))))$. Also, $\boldsymbol{\Gamma}[k]$ can be recursively defined as $\boldsymbol{\Gamma}[k] = \xi_k(\boldsymbol{\Gamma}[k-1]),\ \boldsymbol{\Gamma}[0] = \boldsymbol{\Gamma}$. We term $\mathcal{C}[k-1]$ as the configuration of the heterogeneous multi-robot team before $\xi_k$ occurred. Note that, since in this paper, the robots are entrusted with a monitoring task, hereon we use the terms ``resource'' and ``sensor'' interchangeably in this manuscript.

We are now in a position to formally define the two main problems addressed in this paper:
	
\begin{problem}
		\label{prob: comm graph generation}
		\textbf{Task optimal communication graph generation:} Given a tolerable resource failure $\xi_k$, an associated feasible resource matrix $\boldsymbol{\Gamma}[k]$ and a heterogeneous multi-robot team configuration $\mathcal{C}[k-1]$ find a new communication graph $\mathcal{G}_c[k]$ such that,
		\begin{enumerate}
			\item $\mathcal{G}_c[k]$ is a connected graph,
			\item the multi-robot {monitoring} task performance {metric $J(J_1, J_2, \dots, J_N,\mathcal{C}, T)$} is optimized, and
			\item %the {number of communication links} in new communication graph  $\mathcal{G}_c[k]$ is close to the {number of communication links in} $\Delta$-disk interaction graph $\mathcal{G}[k-1]$.
			\final{\textbf{near topology constraint}: the absolute value of the difference in communication links between the new communication graph $\mathcal{G}_c[k]$ and the number of communication links in $\Delta$-disk interaction graph  $\mathcal{G}[k-1]$ is less than a pre-defined value $\epsilon$.}
		\end{enumerate}
\end{problem}

\final{The condition of graph connectivity in \hyperref[prob: comm graph generation]{Problem~\ref{prob: comm graph generation}} is enforced to ensure cooperative task performance of the team. The next condition guarantees post-failure optimal task performance of the team for the available resources.  Since in many applications it is desirable to have communication networks with as few links as possible, the final condition keeps a check on the  number of additional communication links added to generated graph. }
	
\begin{problem}
		\label{prob: reconfiguration execution}
		\textbf{Reconfiguration Execution:} Given a heterogeneous multi-robot team communication graph $\mathcal{G}_c[k]$, generate feedback control laws that drives the robots to physical locations such that, their associated interaction graph $\mathcal{G}({\bf x})$ spans  $\mathcal{G}_c[k]$. We describe this problem more precisely in~\autoref{sec: config assemble}.
\end{problem}

When a robot team experience a resource failure, information about the failure is transmitted to a base station. The base station checks if the resource failure is tolerable or catastrophic. If catastrophic, it commands the robots to return to their base station as the team no longer retains the necessary resource to perform the assigned task. If the resource failure is tolerable, the base station uses the algorithm developed to solve the optimization problem  \final{described in the upcoming Section}, to generate a new communication graph that optimizes the task performance of the team. Once a new communication graph is created, the base station broadcasts the current and desired configurations to the robots. This information is used by a local \textit{finite-time convergence barrier certificates} feedback controller that drives each robot to a position which realizes a $\Delta$-disk interaction graph that spans the desired communication graph. \autoref{fig:schematic} gives a schematic illustration of our multi-robot $\Delta$-disk interaction graph reconfiguration strategy in the event of a tolerable resource failure. The following sections describes our solution procedure for these two problems. 

\begin{figure*}[!t]
		\centering
		\vspace{2mm}
		\hspace{3mm}
		\includegraphics[width=\linewidth]{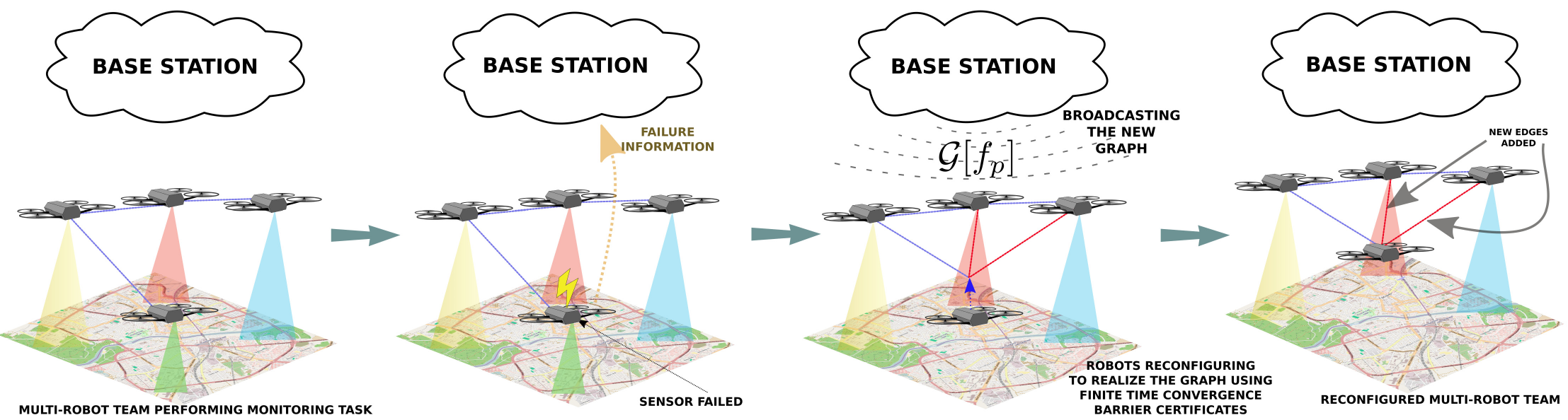}
		\caption{Basic strategy of our approach. When a resource is lost, the base station selects edges to modify the communication graph. Next, the robots use \textit{finite-time convergence barrier certificates} based feedback control laws to drive themselves to positions which realize a $\Delta$-interaction graph that spans the desired communication graph. }
		\label{fig:schematic}       % Give a unique label
	\end{figure*}

\section{Task Performance based Communication graph generation}
\label{sec: opt reconfig}

%\todo{there is cost function J, among all configuration }

In this section, we outline our proposed solution to \autoref{prob: comm graph generation} which is to solve the following optimization problem.

\begin{align}
    \label{eqn:abs:obj}
    \maximize_{\mathcal{C} \in \mathcal{\Bar{C}}} \quad &  J(J_1, J_2, \cdots, J_N,\mathcal{C}, T)\\
    \label{eqn:abs:grph cnnt}
    \text{subject to} ~~ &\lambda_c(\mathcal{G}) > 0\\
    \label{eqn:abs:nr top}
    ~~& |\,|\mathcal{E}|-|\mathcal{E}[k-1]|\,| \leq \epsilon, \, \epsilon > 0,
    %~~&d(\mathcal{G}, \mathcal{G}[k-1]) \leq \epsilon, \, \epsilon > 0,
\end{align}
%where $d(\cdot,\cdot)$ is any reasonable metric defined on the space of graphs with $N$ vertices. 
%With regard to \autoref{prob: comm graph generation}, 
{This abstract optimization problem is a formal mathematical representation of %maximizes the \rag{monitoring} task performance of 
\autoref{prob: comm graph generation}.} %As mentioned earlier, in this article, we demonstrate the effectiveness of the framework through its application on a multi-robot team that is tasked with monitoring a process of interest. 
Next, we describe the details of the monitoring task and the systematic way through which the abstract optimization is extended into a method to handle robot resource failures in the multi-robot team performing the task. 
 %Before we dive into the details of the algorithm, we introduce and describe the notion of \textit{one-hop observability} which forms the basis of our multi-robot task performance metric. 

\subsection{Monitoring task} \label{subsec:monitoringTask}

We now formally define the monitoring task executed by the multi-robot team.  The objective for the robots is to individually estimate the state of an exogenous linear process with state $\mathbf{{e}} \in \mathbb{R}^{n_e}$
\begin{align} \label{eq:processDynamics}
    \mathbf{\Dot{e}} = \mathbf{A}_e\,\mathbf{{e}} + \mathbf{B}_e\,\mathbf{u}_e + {\bf w}_e, 
\end{align}
where $\mathbf{A}_e$ and $\mathbf{B}_e$ are the process and control matrices of the exogenous process respectively, while ${\bf w}_e$ is the zero-mean independent Gaussian noise process. Every robot in the team observes the exogenous process of interest and obtains measurements according to the linear measurement model.
\begin{align}
     \mathbf{y}_i = \mathbf{H}_i\,\mathbf{{e}} +  \boldsymbol{\nu}_i \quad i=1,\dots,N,
\end{align}
where $\mathbf{H}_i$ is the measurement matrix for robot $i$ and while $\boldsymbol{\nu}_i$ is the zero-mean independent Gaussian noise corresponding to the set of measurements of robot $i$. Additionally it is assumed that, the pair $(\mathbf{A}_e,\mathbf{H}), \mathbf{H} = [\mathbf{H}_1^\top\, \mathbf{H}_2^\top\, \cdots\, \mathbf{H}_N^\top]^\top$ is observable. In other words, the multi-robot team is collectively observable. Note that, the individual pairs $(\mathbf{A}_e,\mathbf{H}_i), \forall i=1,\dots,N$ may not be observable. 

In the context of a monitoring task, the resource matrix $\boldsymbol{\Gamma}$ is interpreted as a sensor matrix. Since the sensors are assumed to be linear (linearized for non-linear sensors), we can associate a one dimensional output matrix with each sensor and $\mathcal{H} = \{\mathbf{h}_1, \mathbf{h}_2, \cdots, \mathbf{h}_i, \cdots, \mathbf{h}_r \}, \mathbf{h}_i \in \mathbf{R}^{n_e}$ denote  the set containing all possible sensor-wise output matrices. Next, we show that given a sensor matrix $\boldsymbol{\Gamma}$ and the set $\mathcal{H}$ the output matrix $\mathbf{H}_i$ of robot $i$ can be constructed. If $\mathcal{Y}(i) = \{\gamma^1(i), \gamma^1(i), \cdots, \gamma^{|\mathcal{Y}(i)|}(i) \}$ denote the set of column indices of the non-zero entries of $\boldsymbol{\Gamma}$'s $i^{\text{th}}$ row, then the measurement matrix $\mathbf{H}_i$ can be constructed as $\mathbf{H}_i = [\mathbf{h}_{\gamma^1(i)} \mathbf{h}_{\gamma^2(i)} \cdots \mathbf{h}_{\gamma^{|\mathcal{Y}(i)|}(i)}]^\top$. Since the robots monitor a process we use an observability-based metric to quantify the team task performance. %We delineate the details about the metric used in the paper in the forthcoming subsection. %Note that, since the robots are entrusted with a monitoring task, hereon we use the terms ``resource'' and ``sensor'' interchangeably in this manuscript.

\final{ For ease of readability we explicitly state the assumptions made in this problem formulation:
 \begin{assumption}
 For ease of presentation all robots are assumed to estimate the same process. It is straightforward to verify that this assumption does not impact the generality of our formulation.
 \end{assumption}
 \begin{assumption}
 No two sensor failures occur simultaneously. 
 \end{assumption}
 \begin{assumption}
 The sensors in the multi-robot system are linear.
 \end{assumption}
 \begin{assumption}
 The multi-robot team is  collectively observable.
 \end{assumption}
 \begin{assumption}
 %Although it is assumed that each robot in the team observes a continuous-time process, the framework extends to the case of a discrete-time process as well.
 It is assumed that each robot in the team observes a continuous-time process. Note that the framework extends to the case of a discrete-time process as well. 
 \end{assumption}
 \begin{assumption}
 It is also assumed that subtask allocation is centralized, and the subtasks are assigned before the multi-robot monitoring task execution.
 \end{assumption}
%  \begin{itemize}
%     \item For ease of presentation all robots are assumed to estimate the same process. It is straightforward to verify that this assumption does not impact the generality of our formulation.
%     \item It is assumed that no two sensor failures occur simultaneously. 
%     \item The sensors in the multi-robot system are assumed to be linear. 
%     \item The multi-robot team is assumed to be collectively observable. 
%     \item Although it is assumed that each robot in the team observes a continuous-time process, the framework extends for the case of a discrete-time process as well.
%     \item It is also assumed that subtask allocation is centralized, and the subtasks are assigned before the multi-robot monitoring task execution.
% \end{itemize}
}

\subsection{One-hop observability}
\label{subsec: OHO}

%\ppnote{[I think the following paragraph contrasts a little bit with the story we are telling. In addition, it requires requires the application section to be designed around an {\it agile process}. It can be done, but currently that's not the case.]} %\ppsout{It is a well established fact that the state of any linear process can be estimated in a distributed fashion using a team of robots employing consensus protocols if (i) robot team is collectively observable and (ii) their underlying communication graph is connected \cite{olfati2007dkf}. A size amount of work had been done on distributed state estimation \cite{olfati2007dkf,REGO201936}. A downside of consensus based protocol is its asymptotic rate of convergence \cite{FB-LNS}. The asymptotic convergence of the consensus protocols makes the distributed estimation strategies ill-suited for state estimation of agile processes. Hence, it is apparent that consensus based distributed strategies are inadequate for estimating the state of an agile process.}

%\rag{the story is : robots perform individual estimate for some reason.}
In the context of this paper, we consider each robot in the team executing an individual monitoring task using its locally available measurements. In addition, although one could design distributed estimation schemes (\eg \cite{olfati2007dkf}), the asymptotic convergence of the consensus protocol makes the distributed estimation strategies ill-suited for state estimation of agile processes~\cite{FB-LNS} in large robotic networks. %\ppnote{\sout{Hence, it is apparent that consensus-based distributed strategies are inadequate for estimating the state of an agile process using large robotic networks. [I would take this sentence off, it's too strong. As it is, it should be supported by formal arguments]} }}
%The fact that each robot performs 
%\ppnote{Since in this paper we focus on robots performing} individual monitoring task, they relay on their sensors or its immediate neighbors' sensors to estimate the state of the process of interest.
\final{ Hence in this paper we focus on robots performing individual monitoring task where they rely on their sensors or their immediate neighbors' sensors to estimate the state of the process of interest.} {Consequently, we need a formal way to quantify the state estimation efficacy of a robot to estimate the state of a process using at most the measurements from its own and its immediate neighbors.}

To tackle this problem, we define a new notion of observability which we term as \textit{one-hop observability} ({OHO}). %Recall that,  
%\ppnote{[Maybe we should use a {Definition} environment to emphasize that this a new notion?]}%represents the vertical concatenation of the output matrices of the neighbors of robot $i$  including itself. 

\begin{definition}[One-hop observable robot]
A robot is said to be \textbf{one-hop observable} if the subsystem containing the robot and its one-hop neighbors is observable. Formally, if $\mathbf{H}_i$ is the output matrix associated with robot $i$ then the robot is one-hop observable if and only if the pair $(\mathbf{A}_e, \mathbf{H}_{{\mathcal{\Bar{N}}(i)}})$ is observable, where $\mathbf{H}_{{\mathcal{\Bar{N}}(i)}} = \left[\mathbf{H}_{i^1}^\top, \mathbf{H}_{i^2}^\top,\cdots, \mathbf{H}_{i^{|{\mathcal{\Bar{N}}(i)}|}}^\top \right]^\top$, ${\mathcal{\Bar{N}}(i)} = \{i^1, i^2, \cdots, i^{|{\mathcal{\Bar{N}}(i)}|}\}$.
\end{definition}

We define the robot team to be one-hop observable if each robot in the team is one-hop observable. An immediate ramification of the team being OHO is the fact that, each robot can estimate the state of the monitored process using information from its one-hop neighbors, without the aid of a consensus protocol.   Therefore, %in our framework, 
in order to guarantee that each robot can execute its task, upon the occurrence of a sensor fault, we need to improve the one-hop observability of the team.

\begin{remark}
One-hop observability is a stronger condition than collective observability. The idea can be expanded to the notion of $k$-hop observability, in which the subsystem consisting of the robot and all  robots within its $k$-hop neighborhood results in an observable system. %For $k>1$,  this is relatively a weaker notion that OHO. 
Collective observability is a special case of $k$-hop observability where $k$ equals the diameter \cite{GodsilRoyle2001} of the graph.
\end{remark}

We term the quantity $\mathbf{\Bar{A}} \boldsymbol{\Gamma}$ as the \textit{one-hop sensor matrix} since it stores information about the sensors accessible to a robot within its one-hop neighborhood. In the event of a robot sensor failure, we aim to reconfigure the multi-robot team $\Delta$-disk interaction graph such that every robot in the team is one-hop observable, \final{while adding as few edges as possible to the robots' interaction graph.} Using standard theorems from linear system theory~\cite{Chen1998}, we infer that a robot is one-hop observable if and only if the observability matrix  associated with the pair $(\mathbf{A}_e, \mathbf{H}_{{\mathcal{\Bar{N}}(i)}})$ is full column rank~\cite{Chen1998}. %\final{\sout{Consequently, the system is one-hop observable if the observability matrix associated each $(\mathbf{A}_e, \mathbf{H}_{{\mathcal{\Bar{N}}(i)}}), \forall i \in [N]$ is full column rank.}} 
Since, in general, rank-constrained optimization problems are  classified as NP-hard, due to the discontinuous and non-convex nature of the rank function~\cite{SUN2017128}, developing exact rank-constraint optimization strategies to render a one-hop observable multi-robot team  is difficult. 

%\todo{Add energy related explanation of why do we maximize the trace of the Gramian.} 
Towards this end, we aim to achieve this goal by minimizing the average over the trace of the inverse of the robots' one-hop observability Gramians. Metrics such as trace, determinant and condition number of observability Gramians and their inverses have been extensively  used solving sensor selection problems (see~\cite{hinson2014observability,ilkturk2015observability,Dilip2019} and the references therein). \final{The observability Gramian of a system quantifies its ability to estimate its state through its outputs. Each eigenvalue of the observability Gramian measures the output energy of a particular observable mode associated with the system. Increase in the output energy of an observable mode implies increase in its ease of observability. Hence, the observability Gramian's eigenvalues encode the relative ease of observing the corresponding observable modes.  The minimum eigenvalue of the observability Gramian quantifies the output energy associated with the least observable mode, while the maximum eigenvalue measures the output energy for the most observable mode \cite{hinson2014observability}. Since the trace of the observability Gramian is equivalent to the sum of its eigenvalues, the trace  measures total output energy over all observable modes of the system. Therefore the ease of observability of the system can be maximized by maximizing the trace of observability Gramian or minimizing the trace of the inverse of observability Gramian \cite{hinson2014observability}. } {Here, the subtask of each robot in the team is to estimate the state of the process. Its performance can be measured using the trace of the inverse of its one-hop observability Gramian; lowering the trace of the inverse one-hop observability Gramian increases the subtask performance.  Consequently, the multi-robot task performance is measured as the average of the trace of the inverse of the Gramians.}

Let $\boldsymbol{\Theta}_i(0,T)$ be the $T$ time continuous observability Gramian associated with sensor $i$, defined as: 

\begin{align}
    \label{eqn: snsr obs gram}
    \boldsymbol{\Theta}_i(0,T) = \int^{T}_{0} \exp{(\mathbf{A}_e^\top \tau)}\mathbf{h}_i \mathbf{h}_i^\top \exp{(\mathbf{A}_e \tau)} d \tau.
\end{align}
If,
%\textbf{Continuous Gramian}
\begin{align}
    \boldsymbol{\Theta}(0,T) = \int^{T}_{0} \exp{(\mathbf{A}_e^\top \tau)}\mathbf{H}^\top \mathbf{H} \exp{(\mathbf{A}_e \tau)} d \tau,
\end{align}
where $\mathbf{H} = [\mathbf{h}_1 \mathbf{h}_2 \cdots \mathbf{h}_r]^\top$
then,
\begin{align}
    \boldsymbol{\Theta}(0,T) &= \int^{T}_{0} \exp{(\mathbf{A}_e^\top \tau)}(\sum^r_{i=1} \mathbf{h}_i \mathbf{h}_i^\top) \exp{(\mathbf{A}_e \tau)} d \tau \\
    &= \sum^r_{i=1} \int^{T}_{0} \exp{(\mathbf{A}_e^\top \tau)}\mathbf{h}_i \mathbf{h}_i^\top \exp{(\mathbf{A}_e \tau)} d \tau \\
    &= \sum^r_{i=1} \boldsymbol{\Theta}_i(0,T).
\end{align}

It is straightforward to derive that the one-hop observability Gramian of robot $i$ can be computed as 
\begin{align}
    \label{eqn: one hop obser gram cnstrct}
    \mathbf{O}_i(0,T) = \sum^r_{j=1}[\mathbf{\Bar{A}}\,\boldsymbol{\Gamma}]_{i,j} \boldsymbol{\Theta}_j(0,T).
\end{align}

Recall that $\mathbf{\Bar{A}}$ is the unweighted closed adjacency matrix associated with the communication graph of the robot team. The above equation can be compactly written for the whole multi-robot team as,
\begin{align}
\label{eqn: one hop kron}
    (\mathbf{\Bar{A}}\,\boldsymbol{\Gamma})\otimes \mathbf{I} \begin{bmatrix} \boldsymbol{\Theta}_1(0,T) \\
	\boldsymbol{\Theta}_2(0,T) \\
	\vdots \\
	\boldsymbol{\Theta}_r(0,T)\end{bmatrix} = \begin{bmatrix} \mathbf{O}_1(0,T) \\
	\mathbf{O}_2(0,T) \\
	\vdots \\
	\mathbf{O}_N(0,T)\end{bmatrix}
\end{align}
 where $\mathbf{O}_1(0,T), \mathbf{O}_2(0,T), \cdots, \mathbf{O}_n(0,T)$ matches the one-hop observability Gramians of the corresponding robot in the multi-robot team.

\subsection{Algorithmic solution}

\autoref{algo: comm_grph} presents the  pseudocode for handling robot resource failure in a team performing a monitoring task. Our algorithmic solution relies on a mixed integer semi-definite program (MISDP). Researchers have extensively used MISDPs to formulate and solve graph construction problems~\cite{rafiee2010optimal,Grob2011}.%\rag{\sout{and  thus is an ideal tool to frame our problem}}. 
In essence, the MISDP is a reformulation of the abstract optimization problem (\autoref{eqn:abs:obj}-\autoref{eqn:abs:nr top}) tailored to the requirements of the monitoring task. We now mathematically state the MISDP; in the following discussion we argue that it captures the aspects of \autoref{prob: comm graph generation} in the context of the monitoring task.

\begin{align}
	\label{eqn: MISDP obj}
	\minimize_{\substack{\mathbf{L} \in \mathcal{S}^n_+,\ \mu \in \mathbf{R}_{> 0}, \\ \boldsymbol{\Pi} \in \{0,1\}^{n\times n} \boldsymbol{\Delta},\mathbf{\Bar{O}} \in  \mathcal{S}_+}} \quad & \frac{\text{Trace}(\boldsymbol{\Delta})}{N} \\%+ \nu \|\boldsymbol{\Pi} - \mathbf{\Bar{A}}[k-1] \|_F^2\\
	\label{eqn:cons:zero sum}
	\text{subject to} ~~ & \mathbf{L}\,\mathbf{1}^N = \mathbf{0}^N\\
	\label{eqn:cons:connectivity}
	~~ &\frac{1}{N} \mathbf{1} \mathbf{1}^\top  + \mathbf{L} \succeq \mu \mathbf{I} \\
	\label{eqn:cons:Schur}
	~~ & \begin{bmatrix} \boldsymbol{\Delta} & \mathbf{I} \\
	\mathbf{I} & \mathbf{\Bar{O}}
	\end{bmatrix} \succeq \mathbf{0} \\
	\label{eqn:cons:one hop observability}
	~~ & (\boldsymbol{\Pi}\cdot\boldsymbol{\Gamma}[k])\otimes \mathbf{I}\begin{bmatrix} \boldsymbol{\Theta}_1(0,T) \\
	\boldsymbol{\Theta}_2(0,T) \\
	\vdots \\
	\boldsymbol{\Theta}_r(0,T)\end{bmatrix} = \begin{bmatrix} \mathbf{O}_1 \\
	\mathbf{O}_2 \\
	\vdots \\
	\mathbf{O}_N\end{bmatrix} \\
% 	\label{eqn:cons:sparse resource}
% 	~~ & (\boldsymbol{\Pi}\cdot\boldsymbol{\Gamma}_r) \mathbf{1}^r \leq r\mathbf{1}^n \\
	\label{eqn:cons:diag binary}
	~~ & diag(\boldsymbol{\Pi}) = \mathbf{1}^n \\
	\label{eqn:cons:binary sym}
	~~ & \boldsymbol{\Pi} = \boldsymbol{\Pi}^T\\
	\label{eqn:cons:Lap diag}
	~~ & [\mathbf{L}]_{i,i} > 0 ~\forall\ i \in [N]\\
	\label{eqn:cons:Lap off diag min}
	%~~ & [\mathbf{L}]_{i,j} \geq  -\boldsymbol{\Pi}_{i,j}  ~\forall~(i, j) \in [N]^2,~i \neq j \\
	%\label{eqn:cons:Lap off diag max}
	~~ & [\mathbf{L}]_{i,j} =  -\boldsymbol{\Pi}_{i,j} ~\forall ~(i, j) \in [N]^2,~i \neq j\\
	\label{eqn:cons:topology near}
	~~ & \|\boldsymbol{\Pi} - \mathbf{\Bar{A}}[k-1] \|_F^2 \leq 2\times e,
\end{align}
where $\mathbf{\Bar{O}} = \mathbf{O}_1 \oplus \mathbf{O}_2 \cdots \oplus \mathbf{O}_n$. \final{The decision variable $\mathbf{L}$ is a weighted Laplacian of a graph that has the same topology as $\boldsymbol{\Pi}$ and $\Bar{\mathbf{A}}[k]$. \hyperref[eqn:cons:Lap off diag min]{Constraint~\ref{eqn:cons:Lap off diag min}} guarantees that the off-diagonals of the Laplacian is negative if and only if the corresponding elements in $\boldsymbol{\Pi}$ is unity. \hyperref[eqn:cons:zero sum]{Constraint~\ref{eqn:cons:zero sum}} encodes the property of a graph Laplacian that each row sums to zero (\autoref{eqn:lap property 1}). %\sout{Since the Laplacian associated with every connected graph has only one zero eigenvalue with  $\mathbf{1}$ as its corresponding eigenvector, it can be shown that an undirected graph is connected if and only if $\mathbf{L} + \frac{1}{n} \mathbf{1} \mathbf{1}^T$ is a positive-definite matrix. Hence,} 
Recall from \autoref{sec: notations} that, \hyperref[eqn:cons:connectivity]{Constraint~\ref{eqn:cons:connectivity}} ensures that the generated graph is connected. \hyperref[eqn:cons:diag binary]{Constraint~\ref{eqn:cons:diag binary}} and \hyperref[eqn:cons:binary sym]{Constraint~\ref{eqn:cons:binary sym}}  models $\boldsymbol{\Pi}$ as the closed adjacency matrix of a graph. \hyperref[eqn:cons:one hop observability]{Constraint~\ref{eqn:cons:one hop observability}} is a direct result of \autoref{eqn: one hop kron}. Through a straightforward application of \textit{Schur complement}\cite{boyd1994linear} on \hyperref[eqn:cons:Schur]{Constraint~\ref{eqn:cons:Schur}}, we can show that ${\text{Trace}(\boldsymbol{\Delta})}/{N} \geq \frac{1}{N}\sum_{i=1}^N\text{Trace}(\mathbf{O}_i^{-1})$.} Hence, minimizing \autoref{eqn: MISDP obj} minimizes the average of trace of the inverse of the robots' one-hop observability Gramians. \final{The final constraint depicted in  \autoref{eqn:cons:topology near} enforces the near topology constraint (third condition) delineated in \autoref{prob: comm graph generation}. The parameter $e \in \mathbb{Z}^+$ controls the maximum number of additional edges that can be added to the new communication graph.}

% Algorithm 

\let\oldnl\nl% Store \nl in \oldnl
\newcommand{\nonl}{\renewcommand{\nl}{\let\nl\oldnl}}% Remove line number for one line

\begin{algorithm}[t]
    \caption{Communication graph generation }
    \label{algo: comm_grph}
    \SetKwProg{Fn}{Function}{}{}
    \nonl \textbf{Input:} {$\mathbf{\Bar{A}}[k-1]$: Closed adjacency matrix before failure} \\ 
    \nonl \textbf{Input:} {$\boldsymbol{\Gamma}[k]$: Sensor matrix after sensor failure} \\ 
    \nonl \textbf{Input:} {$\mathbf{A}_e$: State transition matrix} \\
    \nonl \textbf{Input:} {$\mathcal{H}$: Sensor output matrix set} \\
    \nonl \textbf{Output:} {$\mathbf{\Bar{A}}[k]$: Closed adjacency matrix which is OHO} \\
    \Fn{Comm\_graph\_gen($\mathbf{\Bar{A}}[k-1]$, $\boldsymbol{\Gamma}[k]$, $\mathbf{A}_e$, $\mathcal{H}$)}{
    \tcp{check if the configuration is collectively observable}
    \If{$(\mathbf{A}_e, \mathbf{H})$ is not observable}{\tcp{return if it is not}\Return ``Infeasible''}
    \tcp{check if the configuration is still OHO after sensor failure }
    \If{is\_team\_OHO($\mathbf{\Bar{A}}[k-1]$, $\boldsymbol{\Gamma}[k]$, $\mathbf{A}_e$, $\mathcal{H}$)}{\tcp{no change is communication graph required}$\mathbf{\Bar{A}}[k] \leftarrow \mathbf{\Bar{A}}[k-1]$\;
    \Return $\mathbf{\Bar{A}}[k]$\;}
    
    \While{True}{
    Solve MISDP in \autoref{eqn: MISDP obj} - \autoref{eqn:cons:topology near} with $e=1$\;
    \tcp{set $\mathbf{\Bar{A}}[k-1]$ to the solution from MISDP}
    $\mathbf{\Bar{A}}[k-1] \leftarrow \boldsymbol{\Pi}$ \;
    \tcp{check if the new communication graph makes the team OHO }
    \If{is\_team\_OHO($\mathbf{\Bar{A}}[k-1]$, $\boldsymbol{\Gamma}[k]$, $\mathbf{A}_e$, $\mathcal{H}$)}{\tcp{solution found}$\mathbf{\Bar{A}}[k] \leftarrow \mathbf{\Bar{A}}[k-1]$\;
    \Return $\mathbf{\Bar{A}}[k]$\;}
    }
   
    }
    \Fn{is\_team\_OHO($\mathbf{\Bar{A}}$, $\boldsymbol{\Gamma}$, $\mathbf{A}_e$, $\mathcal{H}$)}{
    \tcp{flag to check if the team is OHO }
    flag $\leftarrow $ True\;
    \For{$i\leftarrow 1$ \KwTo $N$}{
    \tcp{check if each robot is OHO}
    \If{$(\mathbf{A}_e, \mathbf{H}_{{\mathcal{\Bar{N}}(i)}})$ is not observable}{
    flag $\leftarrow$ False\;
    break\;}
    }
    \Return flag\;
    }
\end{algorithm}

The base station employs the function \textit{Comm\_graph\_gen} in \autoref{algo: comm_grph} to generate the desired configuration. When the function receives a configuration as an input it first checks if the corresponding multi-robot is collectively observable (Line 2). If the system is not collectively observable, the function is terminated as the problem is infeasible. Next, in Line 4 the algorithm checks if the configuration is currently OHO and returns the inputted configuration if the condition is satisfied. On the contrary, if the inputted configuration fails to be OHO, the function computes the new configuration using the instructions in the loop detailed in Line 7 - Line 12. During each iteration of the loop, a new configuration is generated by solving the MISDP (\autoref{eqn: MISDP obj}-\autoref{eqn:cons:topology near}) with $e=1$ and inputted configuration as parameters and, checks if the generated configuration is OHO (Line 10). If the generated configuration is OHO then it is returned as the solution. Otherwise, the loop reiterates to solve the MISDP with $e=1$ and the generated configuration as the parameters. The following proposition shows that the function \textit{Comm\_graph\_gen} in \autoref{algo: comm_grph} is complete. 

\begin{proposition}
The function \textit{Comm\_graph\_gen} in \autoref{algo: comm_grph} is guaranteed to output a communication graph that ensures that each robot in the multi-robot team is OHO if the  configuration passed as input is collectively observable.
\end{proposition}

\begin{proof}
We prove the proposition by considering the two different types of configuration inputs that can be passed to the algorithm. \\
\textbf{Case 1: Collectively unobservable configuration.} If the input configuration is collectively unobservable then a OHO configuration cannot be generated as collectively observability is a weaker notion than OHO. Hence, the algorithm checks the collective observability condition at Line 2 and return that the problem is infeasible. \\
\textbf{Case 2: Collectively observable configuration.} If the input configuration is collectively observable then at Line 4, the algorithm checks if the configuration is still OHO after the failure. If this condition is satisfied then the inputted configuration is OHO and the algorithm terminates after returning the inputted configuration as the solution. On the contrary, if the inputted configuration fails to satisfy the OHO test in Line 4 then, the function searches for a solution using the loop in Line 7. In the worst case, this loop iterates through all possible connected graphs with $N$ vertices. Since the full connected graph is OHO if the associated configuration is collectively observable, it is a solution to the problem and thus the loop is guaranteed to terminate after a finite number of iterations.
\end{proof}

It is noteworthy that, since it is proven that integer programming problems are NP-Hard \cite{Bernhard_book_12}, \autoref{algo: comm_grph} may not have a polynomial time solution since it involves solving MISDPs. 

%\todo{completeness of algorithm and complexity. The algorithm can always find a solution as long as the system is collectively observable.}
%\subsection{Completeness of \autoref{algo: comm_grph}}
%\label{subsec: completeness of algo}

\section{Configuration Assembling}
\label{sec: config assemble}
From the solution to the problem discussed in the previous section, we obtain a desired communication graph $\mathcal{G}_c[k] = \mathcal{G}(\mathcal{V},\mathcal{E}_c[k])$ which guarantees the post-failure ability for the robots to complete the monitoring task. In this section we introduce a feedback controller that drives the robots to a spatial configuration, such that, under the definition of $\Delta$-disk graph in~\autoref{eq:proximitygraph}, the edges of $\mathcal{G}_c[k]$ can be established. In other words, we want the following constraint to be satisfied
\begin{equation} \label{eq:graphSpanning_objective}
    \mathcal{G}(\mathbf{x}(t)) \subseteq \mathcal{G}_c[k] \quad t<\infty.
\end{equation}
From the definition in~\autoref{eq:proximitygraph}, we note that finite-time convergence is required, as asymptotic convergence is not sufficient to guarantee all desired edges of $\mathcal{G}_c[k]$ will be established.

Following the statement of~\autoref{prob: reconfiguration execution}, the objective is to define a feedback controller $\mathcal{U}_i: \mathbb{R}^{dN} \mapsto U_i$ for all $i=1,\dots,N$, such that $\|\mathbf{x}_i(t) - \mathbf{x}_j(t)\| \leq \Delta$ for all $(i,j) \in \mathcal{E}_c[k]$, for some $t < \infty$. We solve this problem by considering the technique proposed in~\cite{pierpaoli2019sequential}, which we adapt to our setting. 

In particular, we define the following output function
\begin{equation}
    h_{ij}(x) = \Delta^2 - \| \mathbf{x}_i - \mathbf{x}_j \|^2 
\end{equation}
and we note that the constraint in~\autoref{eq:graphSpanning_objective} is satisfied if and only if $h_{ij}(x) \geq 0$, for all $(i,j) \in \mathcal{E}_c[k]$. Alternatively, constraint~\autoref{eq:graphSpanning_objective} is satisfied if and only if $\mathbf{x}(t) \in \mathcal{D}$, where
\begin{equation}
    \mathcal{D} = \{ \mathbf{x} \in \mathbb{R}^{dN} \,|\, h_{ij}(x) \geq 0, \, \forall (i,j)\in  \mathcal{E}_c[k] \}.
\end{equation}
In order to achieve finite-time convergence to $\mathcal{D}$, we introduce the following class-$\mathcal{K}$ function
\begin{equation}
	\bar{\alpha}_{\rho,\gamma}(h_{ij}) = \gamma \cdot \,\text{sign}(h_{ij})\, \cdot |h_{ij}|^\rho,	\label{eq:alpha}
\end{equation}
with $\rho \in [0,1)$ and $\gamma>0$, which is continuous everywhere and locally Lipschitz everywhere except at the origin~\cite{bhat2000finite}. From the definition of $\bar{\alpha}_{\rho,\gamma}(h(x))$ it is possible to construct a set of control inputs for robots $i$ and $j$ that guarantees an edge between them will be assembled in finite time
\begin{equation} \label{eq:admissibleInputs}
K_{ij}^{i} = \{u_i\in  {U}_i \,| \,L_{f_i}\,h_{ij} + L_{g_i}\,h_{ij}\,u_i + \frac{\bar{\alpha} _{\rho,\gamma}(h_{ij})}{2} \geq 0 \},
\end{equation}
where $L_{a} b$ denotes the Lie derivative of $b$ with respect to the field $a$. Extending the edge-wise set in~\autoref{eq:admissibleInputs} to all edges in the desired graph $\mathcal{G}_c[k]$, we obtain a team-wise set of control inputs that satisfy~\autoref{eq:graphSpanning_objective} defined as
\begin{equation} \label{eq:allAdmissible}
    K^{i} = \bigcap_{j\in\mathcal{N}_{i,c}} K_{ij}^{i} \quad \forall i=1,\dots,N.
\end{equation}
Thus, as stated by the following theorem, which we report for completeness, by selecting control inputs from the set~\autoref{eq:allAdmissible}, robots will reach the spatial configuration that satisfies constraint~\autoref{eq:graphSpanning_objective} in finite time.

\begin{theorem} \cite{pierpaoli2019sequential}
\label{thm:fcbfControl_dist}
	Denoting with $\mathbf{x}_0 = [\mathbf{x}_{0,1}^\top,\dots,\mathbf{x}_{0,N}^\top]^\top$ the initial state of a multi-robot system with dynamics~\autoref{eq:affinecontrol}, any controller $\mathcal{U}_i$ such that
	\begin{equation}
	    \mathcal{U}_i(\mathbf{x}_{0}) \in K^{i} \quad \forall \,\mathbf{x}_{0} \in \mathbb{R}^{d|\mathcal{\Bar{N}}_{i,c}|},
	\end{equation} 
	will drive the ensemble state to $\mathcal{D} $ within finite time.
\end{theorem}

Among the many possible strategies to select control inputs from~\autoref{eq:allAdmissible}, we guarantee the least control effort by constructing the following feedback controller
\begin{equation} \label{eq:minQP2}
\begin{aligned}
&\qquad \mathbf{u}_i^* = \arg \min_{\mathbf{u}_i \in U_i}  \| \mathbf{u}_i \|^2 &\\
&L_{\mathbf{f}_i}\,h_{ij} + L_{\mathbf{g}_i}\,h_{ij}\,\mathbf{u}_i + \frac{\bar{\alpha}_{\rho,\gamma}(h_{ij})}{2} \geq 0, & \forall j\in \mathcal{N}_{i,c}
\end{aligned}
\end{equation}
In summary, provided that a solution to the quadratic program~\autoref{eq:minQP2} exists for all $i=1,\dots,N$, robots will assemble $\mathcal{G}_c[k]$ in finite time. We refer the reader to~\cite{pierpaoli2019sequential} for a discussion on the feasibility of this problem.

\section{Application}
\label{sec: application}

We demonstrate the implementation of the reconfiguration process proposed in this paper by considering a network of ground heterogeneous mobile sensors estimating the position of a group of quadrotors executing a coordinated motion (\ie the exogenous process). Denoting with $N_d=5$ the number of quadrotors, the collective state of the process is ${\bf e} = [{\bf e}_1^\top,\dots, {\bf e}_{N_d}^\top]^\top$, where ${\bf e}_i \in \mathbb{R}^3$ is the position of the $i^{\text{th}}$ quadrotor. \final{The experiment design choice of letting the ground robots be the team of mobile sensors undergoing the reconfiguration process was driven by the fact that it is easier to visualize their reconfiguration.}

Assuming the rigid-body dynamics of the quadrotors being faster than the dynamics of controller they execute, for the purpose of estimation, we assume their velocity with respect to the inertial frame can be directly controlled. We let the horizontal motion of the quadrotors be governed by a coordinated controller representing a {\it Cyclic-Pursuit} behavior, while the vertical motion is governed by a {\it Leader-Follower} behavior:
\begin{equation} \label{eq:follower}
\begin{array}{l} \begin{aligned}
    \begin{bmatrix}
    \dot{e}_{i,x} \\ \dot{e}_{i,y} 
    \end{bmatrix} &= R(\theta)\begin{bmatrix}
    e_{i-1,x} - e_{i,x} \\ e_{i-1,y} - e_{i,y}  
    \end{bmatrix} \\
    \dot{e}_{i,z}  &= e_{i-1,z} - e_{i,z} \end{aligned}
\end{array}
\quad \forall \, i=2,\dots,N_d
\end{equation}
where $R(\theta) \in SO(2)$ is planar rotation matrix and $\theta \in [0,2\pi)$ depends on the number of quadrotors. In this case, we place the quadrotors on a circle of radius $0.85\,\text{m}$ and we compute $\theta = \frac{\pi}{N_d}=0.6283$~\cite{ramirez2010distributed}. Similarly, the leader's velocity is:
\begin{equation} \label{eq:leader}
\begin{array}{l} \begin{aligned}
    \begin{bmatrix}
    \dot{e}_{1,x} \\ \dot{e}_{1,y} 
    \end{bmatrix} &= R(\theta)\begin{bmatrix}
    e_{N_d,x} - e_{1,x} \\ e_{N_d,y} - e_{1,y}  
    \end{bmatrix} \\
    \dot{e}_{1,z}  &= u_e(t) \end{aligned}
\end{array}
\end{equation}
where $u_e: \mathbb{R}_+ \mapsto \mathbb{R}$ is the input to the leader (a sinusoidal function in the experiment). The motion described by \autoref{eq:follower} and \autoref{eq:leader} can be compactly represented as in \autoref{eq:processDynamics}:
\begin{equation}
    \mathbf{\Dot{e}} = \mathbf{A}_e\,\mathbf{{e}} + \mathbf{B}_e\,{u}_e(t) + {\bf w}_e,
\end{equation}
where $\mathbf{A}_e$ and  $\mathbf{B}_e$ can be derived from inspection of~\autoref{eq:follower} and~\autoref{eq:leader}. 

In order to execute the monitoring task, the ground robots can execute $5$ different types of measurement, such that each measurement is a linear combination of the components of the state of a single quadrotor. For this, we define the following reduced measurement matrices 
\begin{align} \label{eq:measurementModels}
\begin{aligned}
    \mathbf{\Hat{h}}_1 &= [1,\frac{1}{2},0] \qquad &\mathbf{\Hat{h}}_2 &=[0,1,0]  \qquad  &\mathbf{\Hat{h}}_3 &= [0,0,1] \\
    \mathbf{\Hat{h}}_4 &= [0,\frac{1}{2},1] \qquad &\mathbf{\Hat{h}}_5 &= [\frac{1}{2},0,1]. \qquad & &
    \end{aligned}
\end{align}
Considering that the same type of measurement is repeated for each quadrotor, we match the formulation discussed in~\autoref{subsec:monitoringTask} by defining the suite of sensors available to the robots in term of the set of resources $[r] = \{r_1,\dots,r_5\}$, where
\begin{align} \label{eq:resourceApplications}
    \mathbf{r}_k = \begin{bmatrix}
    \mathbf{h}_{k\,N_d - N_d - 1} \\ \vdots \\ \mathbf{h}_{k\,N_d}
    \end{bmatrix} = \mathbf{I}_{N_d} \otimes \mathbf{\Hat{h}}_k,
\end{align}
for $k=1,\dots,5$.%, where $\mathbf{I}_n$ is the identity matrix of size $n$. 
 In summary, given a subset of the total resources $[r] = \{r_1,\dots,r_5\}$, \autoref{eq:resourceApplications} defines the complete set of measurement available to the robots.

\subsection{Results}
We implement the scenario described above using a team of $9$ differential drive robots from the Robotarium~\cite{wilson2020robotarium} and a team of $5$ quadrotors Crazyflie 2.1 nano-copters controlled through the Crazyswarm~\cite{preiss2017crazyswarm} library. \final{Since there is no direct communication between ground and aerial vehicles}, the transfer of information between the Robotarium and the Crazyswarm environments is handled through ROS nodes. The measurements corresponding to the different sensors described in \autoref{eq:measurementModels} are \final{obtained through simulated sensors} using data from the optical tracking system. Finally, the estimate of the quadrotors' position are obtained through iterations of a standard Kalman filter \final{and the continuous time formulation described above are implemented assuming a fixed time step of 0.033 seconds.}
\begin{figure}[ht]
    \centering
    \includegraphics[width=\columnwidth]{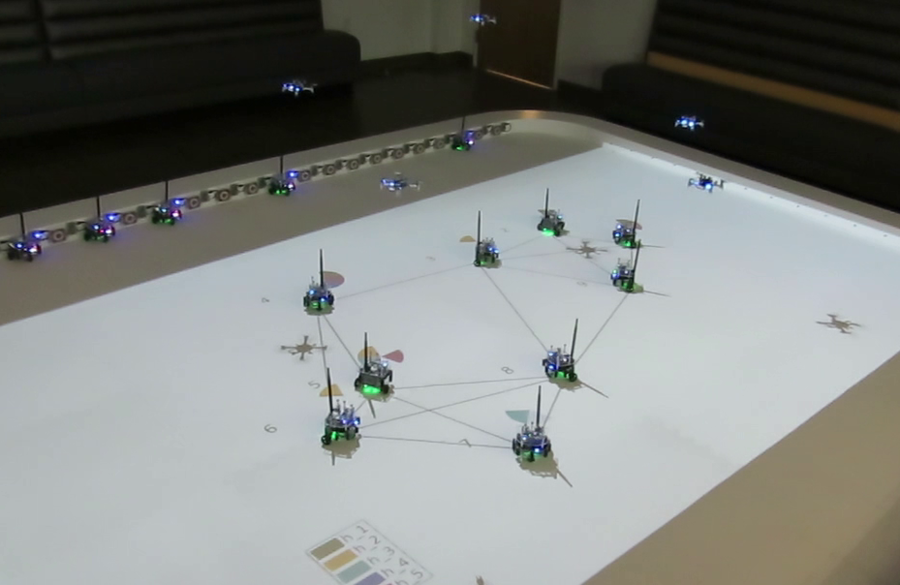}
    \caption{Screenshot during execution of the experiment. A team of $5$ quadrotors executing a coordinated motion represents the exogenous process, that the ground robots are task with monitoring.}
    \label{fig:GazeboDrones}
\end{figure}

With reference to \autoref{fig:5RobotsExp}, we consider a monitoring team with the initial resources distribution represented \autoref{fig:5RobotsExp}(a), subject to the failure sequence described in \autoref{tab:failures} (for clarity, the images in \autoref{fig:5RobotsExp} were obtained from an experiment executed without the quadrotors for the sole purpose of visualizing the motion of the robots during all reconfigurations). 
The pie charts on top of each robot denote which of the five possible sensors each robot is equipped with. At iteration $150$, failure $\xi_1$ occurs, and resource $\boldsymbol{\Gamma}_{7,4}$ is removed (\ie resource $4$ from robot $7$). We observe that, although robot $7$ itself preserves its one-hop observability, robot $8$ loses OHO. Importantly, by \final{measuring the ability of each robot to execute its ask as opposed to simply monitoring resources, robots subject to the loss of a sensor that are still capable of executing their task (e.g., robot $7$ in this case) are not subject to a graph reconfiguration.} From the solution of~\autoref{eqn: MISDP obj}, the additional edge $(5,8)$ is added to the network (red edge in~\autoref{fig:5RobotsExp}(b)), and assembled by the robots (\autoref{fig:5RobotsExp}(c)). The same process is repeated for the removal $\boldsymbol{\Gamma}_{2,3}$ (\autoref{fig:5RobotsExp}(d)) which results in the reconfiguration in \autoref{fig:5RobotsExp}(e). Finally, \autoref{fig:5RobotsExp}(f) depicts the robots configuration after occurrence of all failures.

\begin{figure*}[t]
\centerline{ 
%trim={0cm,0cm,0cm,0cm},
\subcaptionbox{\label{fig2:a}}{\includegraphics[ width=0.66\columnwidth]{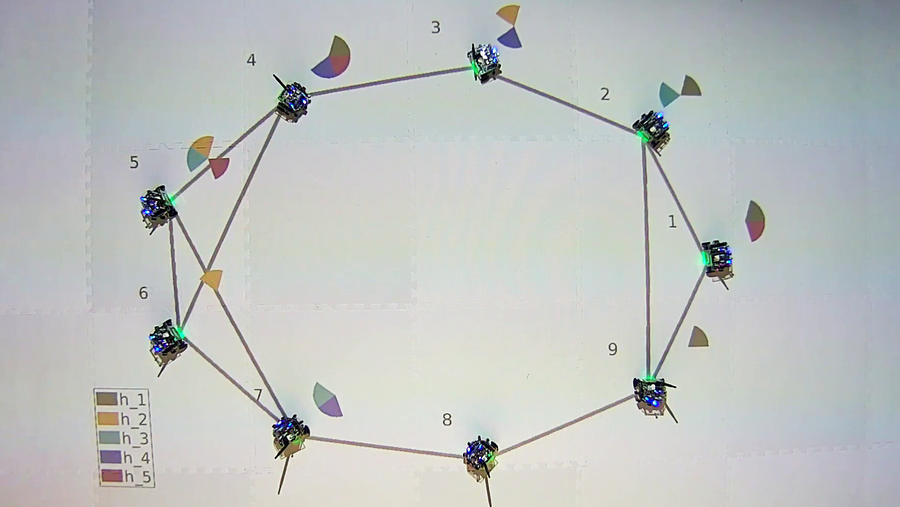}}~
\subcaptionbox{\label{fig2:b}}{\includegraphics[ width=0.66\columnwidth]{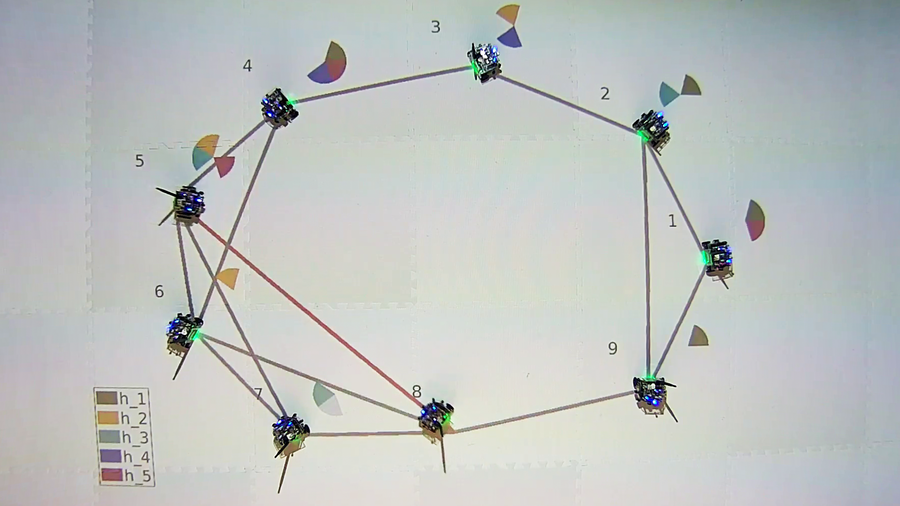}}~
\subcaptionbox{\label{fig2:c}}{\includegraphics[ width=0.66\columnwidth]{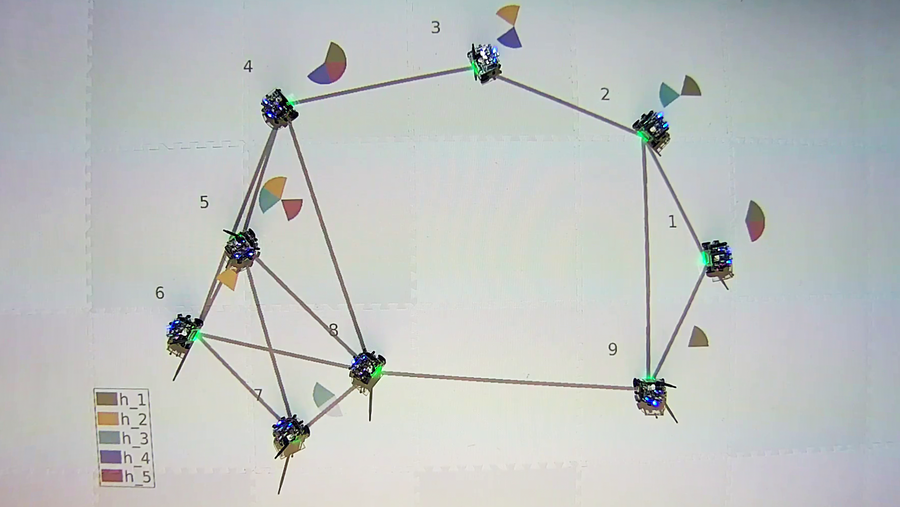}}
} 
\vspace{0.25cm}
\centerline{ 
\subcaptionbox{\label{fig2:d}}{\includegraphics[ width=0.66\columnwidth]{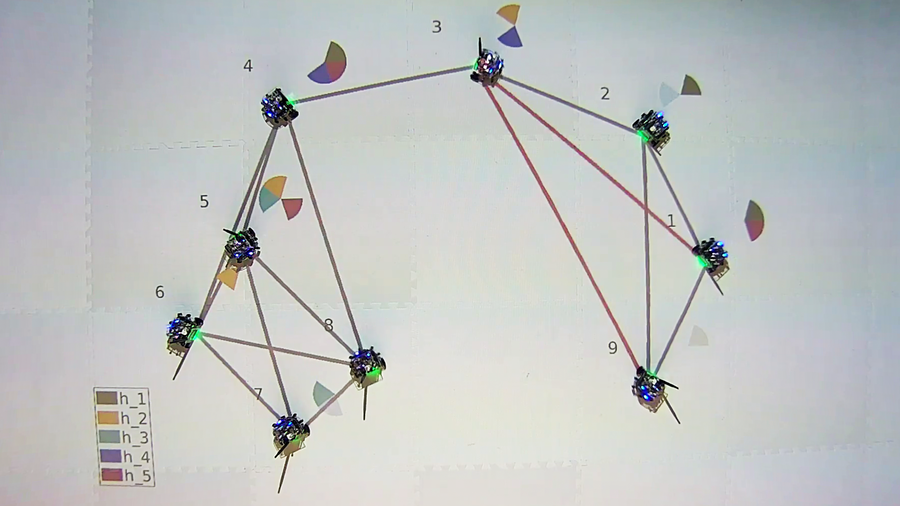}}~
\subcaptionbox{\label{fig2:e}}{\includegraphics[ width=0.66\columnwidth]{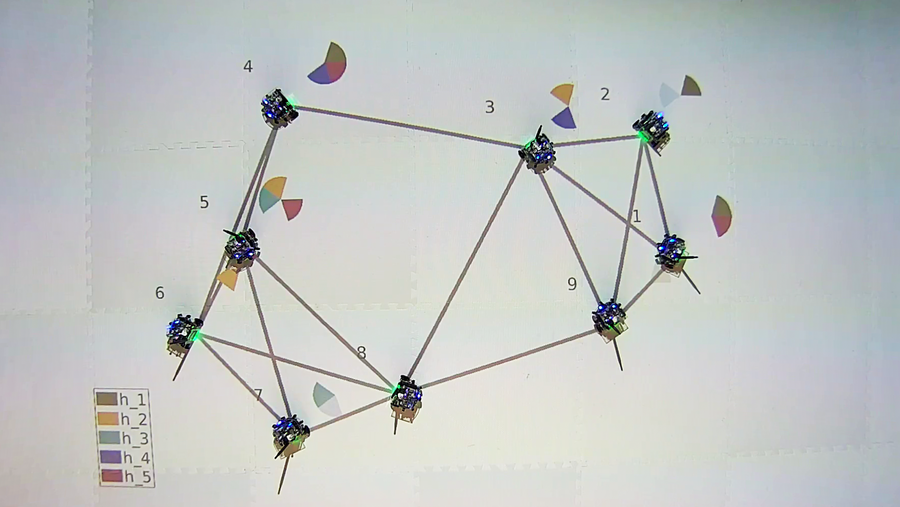}}~
\subcaptionbox{\label{fig2:f}}{\includegraphics[ width=0.66\columnwidth]{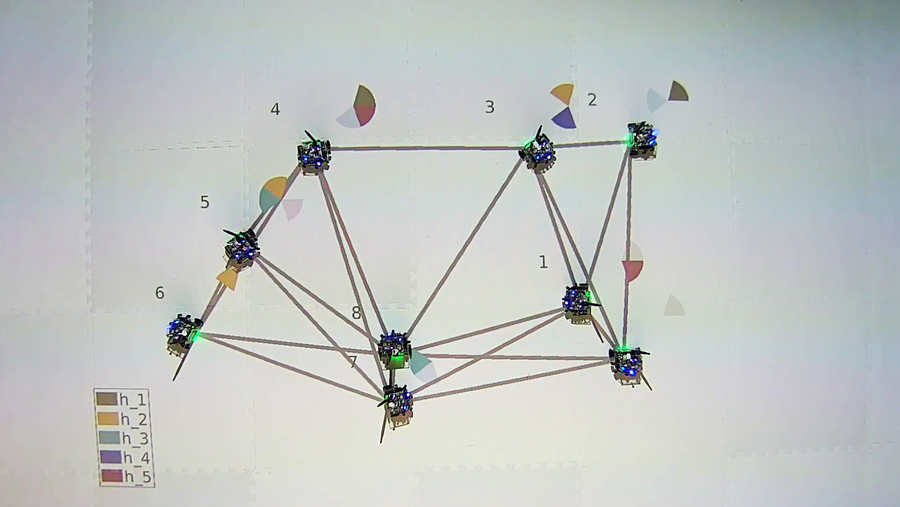}}
}
\caption{Screen-shots for the experiment executed on the Robotarium (ground robots only). The initial communication graph (a) guarantees that all robots satisfy the one-hop observability requirement, given the resources they have access to. Five possible resources are denoted by the pie chart on top of each robot. Due to the first failure, resource $4$ is removed from robot $7$. In accordance with the new configuration robot $8$ needs to access the missing resource from robot $5$ (b) and move towards it in order to establish a communication link (c). Then, because of the removal of resource $3$ from robot $2$ (d), the edges $(1,3)$ and $(3,9)$ are established (e). After all faults $\xi_1$ through $\xi_6$ have occurred, the final configuration of the robots (f) allow them to continue the monitoring task.
\label{fig:5RobotsExp}}
\end{figure*}

\begin{table}[ht]
    \centering
    \begin{tabular}{c|c|c|c  c|c|c|c}
        & Iter. & Resource & Robot & & Iter. & Resource & Robot \\
       %\hline
       $\xi_1$  & 150 & 4 & 7  & $\xi_4$  & 600 & 1 & 1\\
       $\xi_2$  & 300 & 1 & 9  & $\xi_5$  & 750 & 5 & 5\\
       $\xi_3$  & 450 & 3 & 2  & $\xi_6$  & 900 & 4 & 4
    \end{tabular}
    \caption{Resource failures considered in the experiment. Iteration columns represent the iteration of occurrence of failure $\xi_i$, $i=1,\dots,6$.}
    \label{tab:failures}
\end{table}

The impact of the faults on the quality of the monitoring process is shown in \autoref{fig:error_robot23}, where we show the mean estimate error along the three spatial dimensions of the quadrotors' motion, relative to robots $1$ and $8$. The red shadowed areas in \autoref{fig:error_robot23} represent the time interval during which the team went through a reconfiguration. We can observe that network reconfiguration was necessary only for failures $\xi_1$, $\xi_3$, and $\xi_5$. The removal of resource $\boldsymbol{\Gamma}_{7,4}$ and $\boldsymbol{\Gamma}_{5,5}$ (faults $\xi_1$ and $\xi_5$) affects the estimate error for robot $8$ (bottom figure), but does not affect estimate for robot $1$ (upper figure). On the other side, the removal of resource $\boldsymbol{\Gamma}_{2,3}$ (fault $\xi_3$) affects the estimate of robot $1$ which relays on resource $3$ from robot $2$. Finally, the results in \autoref{fig:lambda_23} represent the minimum eigenvalue of the observability Gramian $\lambda_{\mathbf{O}_i}$, for $i=1,\dots,N$. As expected, the observability Gramian becomes singular for both robots $1$ and $8$ under the effects of failures $(\xi_1,\xi_5)$ and $\xi_3$ respectively.
\begin{figure}[ht]
    \centering
    \includegraphics[width=\columnwidth]{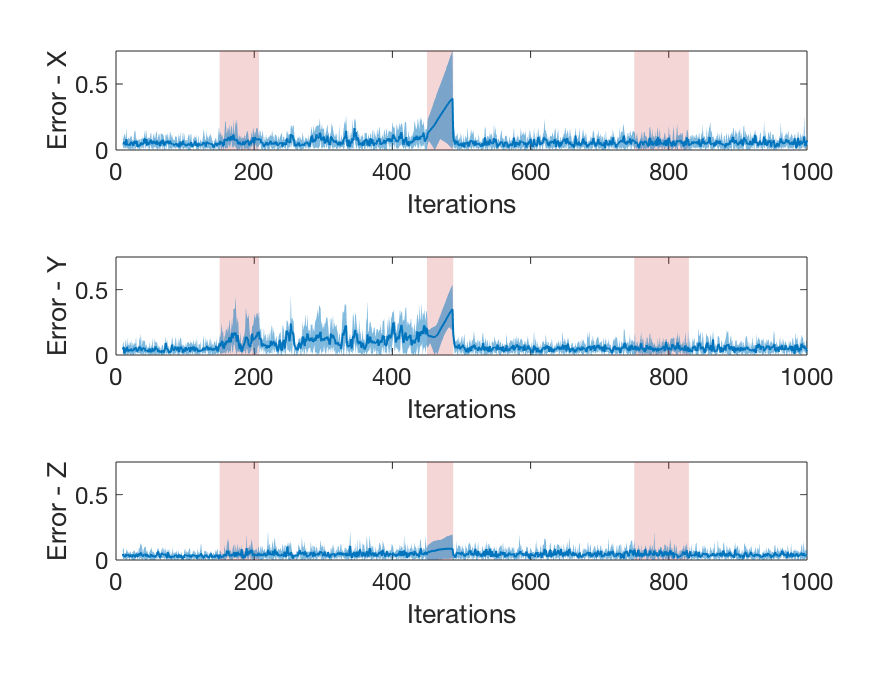} \\
    \includegraphics[width=\columnwidth]{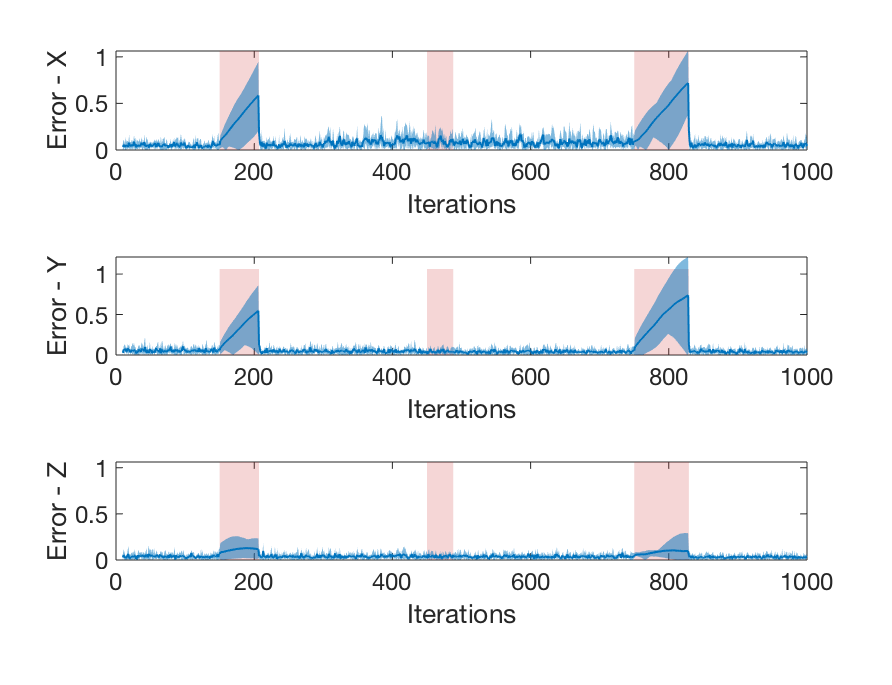} 
    \caption{Mean estimation error for robot $1$ (upper 3 panels) and $8$ (lower 3 panels) over time. The solid lines represent the mean of the estimation error absolute value for each of the three spatial dimensions of the drones' state. The shaded blue area is the interval between the minimum and maximum of each error. The red rectangular patches denote the intervals of time between the occurrence of faults and completion of corresponding reconfiguration. We note the quality of the estimate deteriorating during reconfiguration processes. Note, the removal of resource $3$ from robot $7$ (fault between iterations $150-200$) compromised the one-hop observability for robot $8$ (bottom panels) but did not affect the estimate's quality for robot $1$ (upper panels). We observe a similar effect for the fault at iteration $750$.}
    \label{fig:error_robot23}
\end{figure}

\begin{figure}[ht]
    \centering
    \includegraphics[width=\columnwidth]{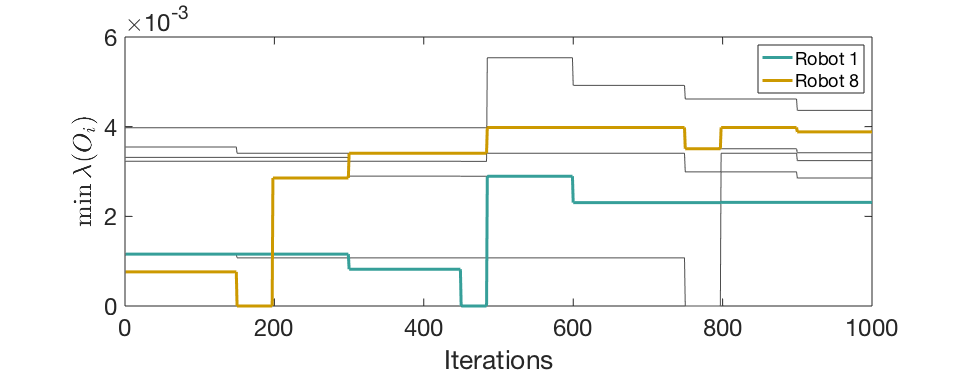}
    \caption{Minimum eigenvalue of the observability Gramian $\mathbf{O}_i$ for $i=1,\dots,9$ over time. Tick curves represent the eigenvalues for robot $1$ and $8$. The observability Gramian becomes singular for both robots $1$ and $8$ under the effects of failure $(\xi_1,\xi_5)$ and $\xi_3$ respectively.}
    \label{fig:lambda_23}
\end{figure}

\final{In order to quantify the benefits of the proposed approach, we compare it against two alternative techniques while considering different sequences of failures. With reference to \autoref{fig:fixed_IC_random_faults}, we assume $20$ sequences of pre-computed randomly selected faults for the team of ground robots considered in \autoref{fig:5RobotsExp}(a). Each sequence of faults removes one resource at a time until no resources are left. The results in \autoref{fig:fixed_IC_random_faults} show the number of faults the team of ground robots was able to tolerate before the pair $(\mathbf{A}_e, \mathbf{H}_{{\mathcal{\Bar{N}}(i)}})$ was not observable at some $i=1,\dots,9$. Along with the network reconfiguration strategy proposed in this paper, we consider a {\it random} alternative technique were edges are created between two randomly selected robots upon occurrence of a fault. A {\it greedy} comparison strategy was also considered. In this case, upon occurrence of a fault, the robot that experiences a resource failure connects to the robot having maximum value of the trace of OHO.
As we can observe from \autoref{fig:fixed_IC_random_faults}, the number of faults tolerated when considering the OHO in the reconfiguration process always equals the number of faults necessary to make the system collective not-observable (black dots). In other words, the OHO reconfiguration technique is capable of maximally exploiting the resources in the team.}

\final{Finally, we compare the same three network reconfiguration strategies for increasing values of initial resource distribution in the team. In particular, we consider mean values of initial resources per robots from $1.2$ to $4$. As we can observe in \autoref{fig:varying_IC_random_faults}, as one would expect, at high density of resources distribution our strategy performs as good as a random approach, i.e., with high resources redundancy failures can be accommodated with very little knowledge about the system. On the other side, when resources are limited, our reconfiguration strategy endows the team with the ability to always tolerate a higher number of faults. }

\begin{figure}[ht]
    \centering
    \includegraphics[width=\columnwidth]{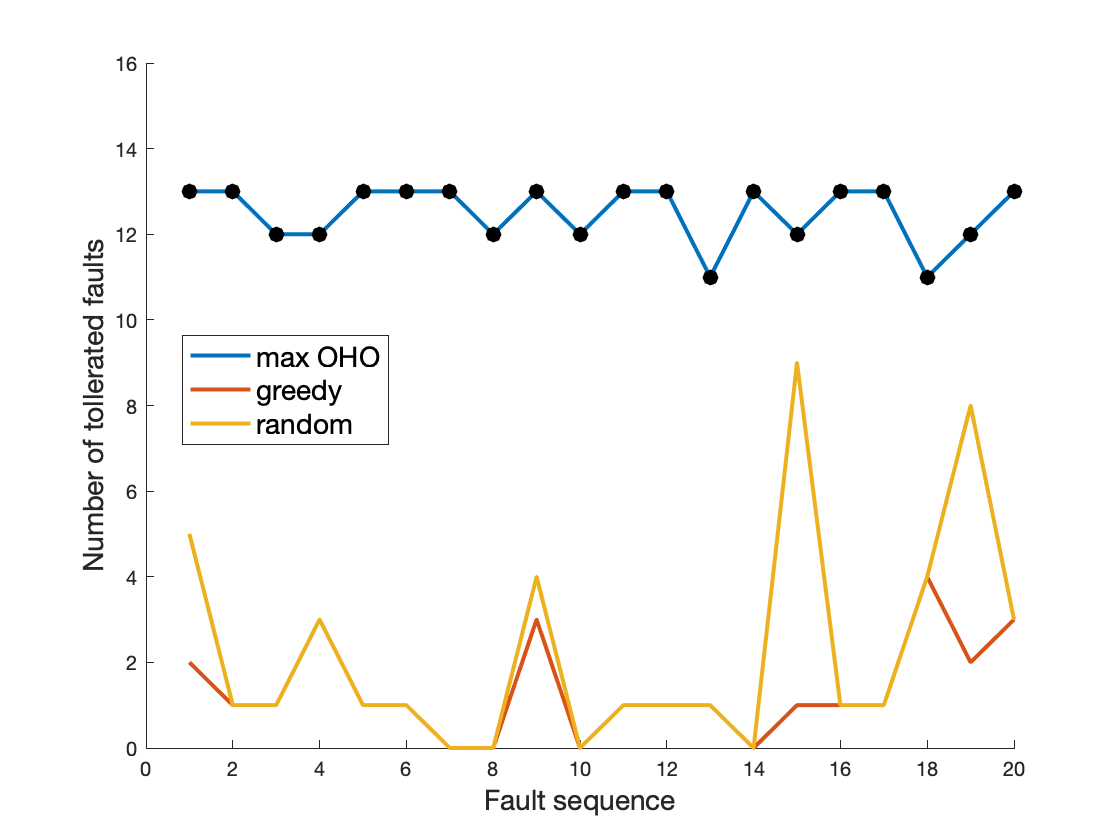}
    \caption{Number of tolerated faults for different reconfiguration strategies and 20 distinct fault sequences. All fault sequences assume the same initial resources distribution among the sensors. The blue line represents the results when implementing the OHO-based technique proposed in this paper, which is compared against random and greedy approach. The black dots in figure represent the number of faults after which the system is no longer collectively observable, i.e., even when connected by a complete graph, all sensors cannot successfully complete their task.}
    \label{fig:fixed_IC_random_faults}
\end{figure}

\begin{figure}[ht]
    \centering
    \includegraphics[width=\columnwidth]{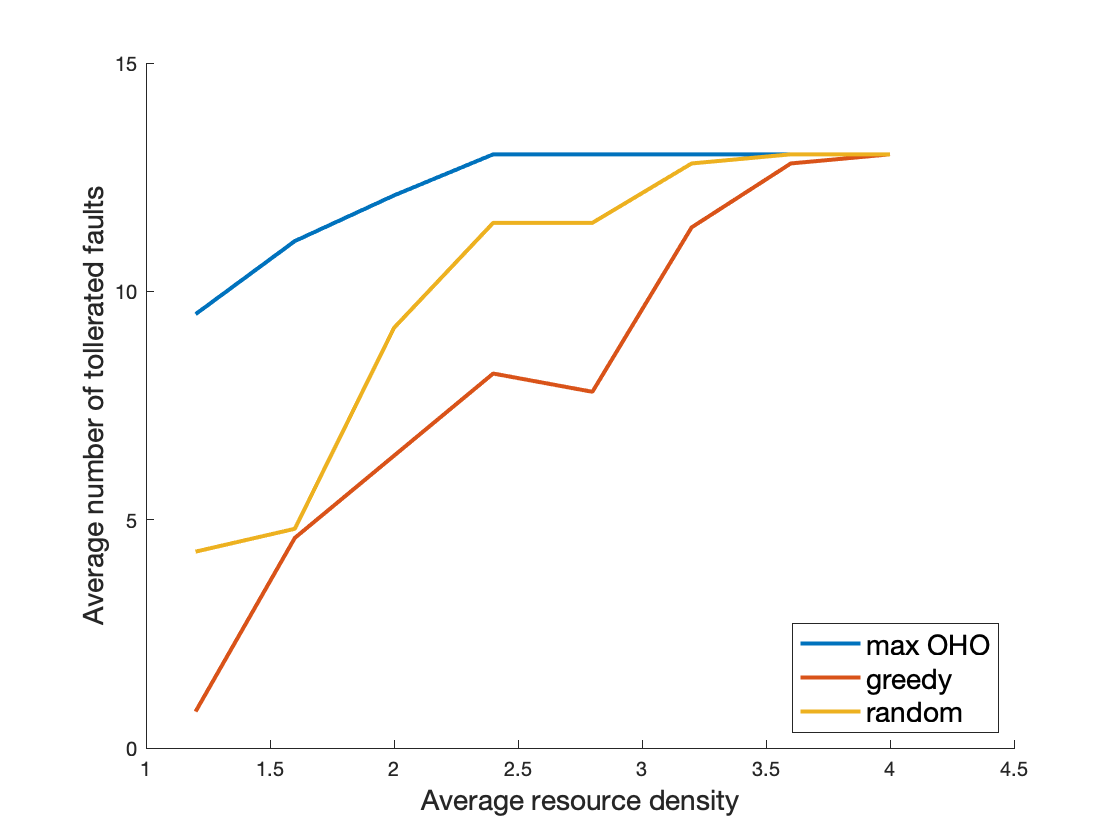}
    \caption{Number of tolerated faults for different reconfiguration strategies and randomly selected initial resource distribution. Mean resources per robot between 1.2 and 4 are considered. For each resource density, 10 distinct initial conditions are selected at random. The number of tolerated faults when implementing the OHO-based reconfiguration technique is always higher then the two other methods.}
    \label{fig:varying_IC_random_faults}
\end{figure}

\section{Conclusion}
\label{sec: conclude}
We described a framework for resilience in a networked heterogeneous multi-robot team subject to resource failures. %we proposed a framework to distribute resources in a heterogeneous team of robots that guarantees resilience to tolerable faults. 
Our framework builds on the novel notion of one-hop observability which allows post-failure reconfiguration of resource sharing, such that all robots can continue their tasks. The proposed framework, composed of two phases, leads to the principled  reconfiguration  of  information  flow  in  the  team  to effectively replace the lost resource on one robot with information from  another,  as  long  as  certain  conditions  are  met. The first phase selects  edges to  be  modified  in  the  system's  communication  graph  after  a resource   failure   has   occurred and the second uses  finite-time  convergence  control  barrier  functions to synthesize feedback control laws to drive each robot to spatial coordinates that  enable  the  communication  links  of  the modified  configuration. We validate the effectiveness of our framework through multi-robot experiments. In heterogeneous teams, the application of standard optimal task-assignment techniques (which include resilient reconfiguration as special case) is difficult due to the lack of computationally amenable representations that map between low-level robot resources to their effect on the overall team task. Through the notion of one-hop observability, we are able \final{to take a step towards a novel paradigm for} task-assignment in heterogeneous multi-robot systems based on the explicit mapping between the resources available to a robot and its performance (\ie the observability Gramian in the context of this paper).

%Through the notion of one-hop observability,we are able to model task-assignment in heterogeneous multi-robot  systems  based  on  the  explicit  mapping  between  theresources  available  to  a  robot  and  its  performance  (i.e.,theobservability  Gramian).  By  the  construction  of  this  formalframework  we  hope  to  lay  a  foundation  for  other  challengesin resilient heterogeneous multi-robot systems

%\bibliographystyle{IEEEtran}
%\bibliography{IEEEabrv,biblio}
% Generated by IEEEtran.bst, version: 1.14 (2015/08/26)

\end{document}